\documentclass{tlp}

\usepackage{amsmath}
\usepackage{graphicx}
\usepackage{multirow}
\usepackage[ruled,vlined,linesnumbered]{algorithm2e}

\usepackage{amsfonts,amssymb}

\usepackage{todonotes}

\usepackage{xspace}
\usepackage{url}

\newcommand{\prg}{P}
\newcommand{\vars}[1]{\mathcal{V}(#1)}

\newcommand{\preds}[1]{\mathcal{P}(#1)}
\newcommand{\heads}[1]{\mathcal{H}(#1)}
\newcommand{\rules}[1]{\mathcal{R}(#1)}
\newcommand{\weak}[1]{\mathcal{W}(#1)}
\newcommand{\HU}{ HU_{\prg}}
\newcommand{\HB}{ HB_{\prg}}

\newcommand{\AS}[1]{AS({#1})}
\newcommand{\OptAS}[1]{OptAS({#1})}
\newcommand{\ASPQ}{ASP(Q)\xspace}

\newcommand{\ASPQW}{ASP$^{\omega}$(Q)\xspace}
\newcommand{\qprg}{\Pi}
\newcommand{\cprg}{C}
\newcommand{\wprg}{C^w}
\newcommand{\fix}[2]{\mathit{fix_{#1}(#2)}}
\newcommand{\fixedqprg}[2]{\qprg_{#1,#2}}
\newcommand{\weakarr}{\leftarrow_{w}}
\newcommand{\weights}{ws}
\newcommand{\levelCost}[3]{\mathcal{C}({#1},{#2},{#3})}

\newcommand{\dffr}{\mathit{diff}}

\newcommand{\lemmacolone}{
\begin{lemma}[Correctness $col_1(\cdot)$ transformation]\label{lemma:uni-plain}
Let program $\qprg$ be such that $n \geq 2$ and the first two subprograms are plain and uniform, i.e., $\Box_1 = \Box_2$, and $\weak{\prg_1} = \weak{\prg_{2}}=\emptyset$, then $\qprg$ is coherent if and only if $col_1(\Pi) = \Box_1 \prg_1 \cup \prg_{2}\Box_{3} \prg_{3}\ldots\Box_n \prg_n:C$ is coherent.
\end{lemma}
}

\newcommand{\obsOr}{
\begin{obs}[ Trivial model existence]\label{obs:or_prg}
    Let $\prg$ be an ASP program, and $l$ be a fresh literal not appearing in $P$, then the following hold: $\{l\}$ is the unique answer set of $or(\prg,l)\cup \{l\leftarrow\}$; 
    and $\AS{or(\prg,l)\cup\{\leftarrow l\}}=\AS{\prg}$.
\end{obs}
}

\newcommand{\defcoltwo}{
\begin{definition}[Collapse notplain-plain existential subprograms]\label{def:col_2}
{Let $\qprg$ be an \ASPQW program of the form $\exists \prg_1 \exists \prg_2 \ldots \Box_n \prg_n: C$, where $\weak{\prg_1}\neq\emptyset$, $\weak{\prg_i}=\emptyset$, with $1<i\leq n$, and $\Box_i \neq \Box_{i+1}$ with $1<i<n$, then:}

    {\footnotesize
    \[
        col_2(\qprg) = \left
        \{\begin{array}{ll}
            \exists \prg_1 \cup or(\prg_2,unsat) \cup W: C \cup \{\leftarrow unsat\} &  n = 2\\
            \exists \prg_1 \cup or(\prg_2,unsat) \cup W\ \forall \prg_3' : C \cup \{\leftarrow unsat\} &  n = 3\\
            \exists \prg_1 \cup or(\prg_2,unsat) \cup W\ \forall \prg_3'\ \exists \prg_4 \cup \{\leftarrow unsat\}\ldots\Box_n \prg_n: C &  n > 3\\
        \end{array}\right.
    \]}
    
    where $W = \{\{unsat\}\leftarrow\} \cup \{\weakarr unsat\ [1@l_{min}-1]\}$, with $l_{min}$ be the lowest level in $\weak{\prg_1}$ and $unsat$ is a fresh symbol not appearing anywhere else, and $\prg_3' = or(\prg_3,unsat)$.
\end{definition}
}
\newcommand{\lemmaColTwo}{
\begin{lemma}[ Correctness $col_2(\cdot)$ transformation]\label{lemma:collapse_ew_e}
    Let $\qprg$ be an \ASPQW program of the form $\exists \prg_1 \exists \prg_2 \ldots \Box_n \prg_n: C$, where $\weak{\prg_1}\neq\emptyset$, $\weak{\prg_i}=\emptyset$, with $1<i\leq n$, and $\Box_i \neq \Box_{i+1}$ with $1<i<n$. Then $\qprg$ is coherent if and only if $col_2(\qprg)$ is coherent.
\end{lemma}
}

\newcommand{\defColThree}{
\begin{definition}[Collapse notplain-plain universal subprograms]\label{def:col_3}
    { Let $\qprg$ be an \ASPQW program of the form $\forall \prg_1 \forall \prg_2 \ldots \Box_n \prg_n: C$, where $\weak{\prg_1}\neq\emptyset$, $\weak{\prg_i}=\emptyset$, with $1<i\leq n$, and $\Box_i \neq \Box_{i+1}$ with $1<i<n$, then:}
    
    {\footnotesize
    \[
        col_3(\qprg) = \left
        \{\begin{array}{ll}
            \forall \prg_1 \cup or(\prg_2,unsat) \cup W: { or(C,unsat)} &  n = 2\\
            \forall \prg_1 \cup or(\prg_2,unsat) \cup W\ \exists \prg_3' : { or(C,unsat)} &  n = 3\\
            \forall \prg_1 \cup or(\prg_2,unsat) \cup W\ \exists \prg_3'\ \forall \prg_4 \cup \{\leftarrow unsat\}\ldots\Box_n \prg_n: C &  n > 3\\
        \end{array}\right.
    \]}
    
    where $W = \{\{unsat\}\leftarrow\} \cup \{\weakarr unsat\ [1@l_{min}-1]\}$, with $l_{min}$ be the lowest level in $\weak{\prg_1}$ and $unsat$ is a fresh symbol not appearing anywhere else, and $\prg_3' = or(\prg_3,unsat)$.
\end{definition}
}
\newcommand{\lemmaColThree}{
\begin{lemma}[Correctness $col_3(\cdot)$ transformation]
    Let $\qprg$ be an \ASPQW program of the form $\forall \prg_1 \forall \prg_2 \ldots \Box_n \prg_n: C$, where $\weak{\prg_1}\neq\emptyset$, $\weak{\prg_i}=\emptyset$, with $1<i\leq n$, and $\Box_i \neq \Box_{i+1}$ with $1<i<n$. Then $\qprg$ is coherent if and only if $col_3(\qprg)$ is coherent.
\end{lemma}
}

\newcommand{\defTranslateWeak}{
\begin{definition}[Transform weak constraints]\label{def:weak_rule} 
    Let $\prg$ be an ASP program with weak constraints, then 
    
    {\footnotesize
    
    \[
        check(P) = \left \{ \hspace*{-0.2cm} 
            \begin{array}{rlr} 
                v_c(w,l,T) \leftarrow & \near b_1,\ldots,b_m  \hspace*{3.6cm} \forall c:\ \ \weakarr b_1,\ldots,b_m [w@l,T] \in P\\
                cl_{\prg}(C,L) \leftarrow & \near level(L), C = \#sum\{w_{c_1};\ldots;w_{c_n}\} & \\
                clone^o(v_c(w,l,T) \leftarrow & \near b_1,\ldots,b_m)  \hfill \forall c:\ \ \weakarr b_1,\ldots,b_m [w@l,T] \in P\\
                clone^o(cl_{\prg}(C,L) \leftarrow & \near level(L), C = \#sum\{w_{c_1};\ldots;w_{c_n}\}) & \\
                \dffr(L) \leftarrow & \near cl_{\prg}(C1,L),  cl_{\prg}^{o}(C2,L), C1\neq C2 &\\
                 hasHigher(L) \leftarrow & \near \dffr(L), \dffr(L1), L<L1 &\\
                 higest(L) \leftarrow & \near \dffr(L),\naf hasHigher(L) &\\
                 dom_{\prg} \leftarrow & \near higest(L), cl_{\prg}(C1,L), cl^{o}_{\prg}(C2,L), C2<C1&\\
            \end{array} \right.
    \]
    }
    \noindent where each $w_{c_i}$ is an aggregate element of the form $W,T : v_{c_i}(W,L,T)$.
\end{definition}
}

\newcommand{\defColFour}{
\begin{definition}[Transform existential not-plain subprogram]\label{def:col_4}
    Let $\qprg$ be an existential alternating \ASPQW program such that all subprograms are plain except the first one (i.e. $\weak{\prg_1}\neq\emptyset$, $\weak{\prg_i}=\emptyset${,} $1 < i \leq n$), then
    
    {\footnotesize
    \[
        col_4(\qprg) = \left
        \{\begin{array}{ll}
            \exists \rules{\prg_1}\forall clone^o(\rules{\prg_1})\cup check(\prg_1): \{\leftarrow dom_{\prg_1}\} \cup C &  n = 1\\
            \exists \rules{\prg_1}\forall \prg_2': \{\leftarrow dom_{\prg_1}\} \cup C &  n = 2\\
            \exists \rules{\prg_1}\forall \prg_2'\exists \prg_3\cup\{\leftarrow dom_{\prg_1}\}\ldots\Box_n \prg_n: C &  n \geq 3\\
        \end{array}\right.
    \] \smallskip }
    
    where $\prg_2' = clone^o(\rules{\prg_1})\cup check(\prg_1) \cup or(\prg_2,dom_{\prg_1})$.
\end{definition}
}
\newcommand{\lemmaColFour}{
\begin{lemma}[Correctness $col_4(\cdot)$ transformation]\label{lemma:rewrite_ew}
Let $\qprg$ be an existential alternating \ASPQW program such that all subprograms are plain except the first,
then $\qprg$ is coherent if and only if $col_4(\qprg)$ is coherent.
\end{lemma}
}

\newcommand{\defColFive}{
\begin{definition}[Transform universal not-plain subprogram]\label{def:col_5}
    Let $\qprg$ be a universal alternating \ASPQW program such that all subprograms are plain except the first one (i.e. $\weak{\prg_1}\neq\emptyset$, $\weak{\prg_i}=\emptyset$ $1 < i \leq n$), then
    
    {\footnotesize
    \[
        col_5(\qprg) = \left
        \{\begin{array}{ll}
            \forall \rules{\prg_1}\exists clone^o(\rules{\prg_1})\cup check(\prg_1): or(C,dom_{\prg_1}) &  n = 1\\
            \forall \rules{\prg_1}\exists \prg_2':  or(C,dom_{\prg_1}) &  n = 2\\
            \forall \rules{\prg_1}\exists \prg_2'\forall \prg_3\cup\{\leftarrow dom_{\prg_1}\}\ldots\Box_n \prg_n: C &  n \geq 3\\
        \end{array}\right.
    \]\smallskip }
    
    where $\prg_2' = clone^o(\rules{\prg_1})\cup check(\prg_1) \cup or(\prg_2,dom_{\prg_1})$.
\end{definition}
}
\newcommand{\lemmaColFive}{
\begin{lemma}[Correctness $col_5(\cdot)$ transformation]\label{lemma:rewrite_fw}
Let $\qprg$ be a universal alternating \ASPQW program such that all subprograms are plain except the first, then 
 $\qprg$ is coherent if and only if $col_5(\qprg)$ is coherent.
\end{lemma}
}

\newcommand{\thmConvergence}{
\begin{thm}[ \ASPQW to \ASPQ convergence and correctness]\label{thm:convergence}
Given program $\qprg$, Algorithm~\ref{alg:cancel_weak} terminates and returns an alternating \ASPQ program $\qprg'$ that is $\qprg'$ is coherent iff $\qprg$ is coherent, and $nQuant(\qprg') \leq nQuant(\qprg) + 1$.
\end{thm}
}

\newcommand{\thmUpperBound}{
\begin{thm}[Upper bound]\label{thm:general_complexity}
The coherence problem of an \ASPQW program $\qprg$ is in: $(i)$ $\Sigma_{n{+1}}^p$ for existential programs, and $(ii)$ $\Pi_{n{+1}}^p$ for universal programs, where $n=nQuant(\qprg)$.
\end{thm}
}
\newcommand{\thmLowerBound}{
\begin{thm}[Lower bound]\label{thm:hardness_general}
The coherence problem of an \ASPQW program is hard for $(i)$ $\Sigma_{n}^p$ for existential programs, and hard for $(ii)$ $\Pi_{n}^p$ for universal programs{, where $n = nQuant(\qprg)$.}
\end{thm}
}

\newcommand{\corLastPlainSigmaN}{
\begin{corollary} [First completeness result]
The coherence problem of an \ASPQW program where the last subprogram is plain (i.e., $\weak{P_n}=\emptyset$) is $(i)$ $\Sigma_{n}^p$-complete for existential programs, and $(ii)$ $\Pi_{n}^p$-complete for universal programs, {where $n = nQuant(\qprg)$.}
\end{corollary} 
}

\newcommand{\thmEEwSigmaTwo}{
\begin{thm}[ Second completeness results]\label{thm:complexity_e_ew}
    Deciding coherence of uniform existential \ASPQW programs with two quantifiers (i.e. $n=2$) such that $\prg_2$ is not plain is $\Sigma_2^p$-complete.
\end{thm}
}

\newcommand{\propEEOrFFisNPorCONP}{
\begin{prop}[Third completeness results]\label{thm:complexity-plain-uniform}
    Deciding coherence of plain uniform \ASPQW programs with $2$ quantifiers is $(i)$ NP-complete for existential programs; and $(i)$ coNP-complete for universal programs.
\end{prop}
}

\newcommand{\thmWeightAndLevelsIsDelta}{
\begin{thm} [Fourth completeness results]\label{thm:models-delta-general}
{Deciding} whether an atom $a$ belongs to an optimal quantified answer set of a plain alternating existential \ASPQW program with $n$ quantifiers is $\Delta_{n+1}^P$-complete.
\end{thm}
}

\newcommand{\thmNoLevelCostantWeightIsTheta}{
\begin{thm} [Fifth completeness results]\label{thm:models-delta}
{ Deciding} whether an atom $a$ belongs to an optimal quantified answer set of a plain alternating existential \ASPQW program with $n$ quantifiers is $\Theta_{n+1}^P$-complete if there is only one level and all the weights are the same.
\end{thm}
}

\newtheorem{thm}{Theorem}
\newtheorem{corollary}{Corollary}[thm]
\newtheorem{lemma}{Lemma}
\newtheorem{obs}{Observation}
\newtheorem{definition}{Definition}
\newtheorem{prop}{Proposition}
\newtheorem{example}{Example}[section]

\newcommand{\myParagraph}[1]{\smallskip\noindent\textit{#1}\xspace}
\newcommand{\myProofSketch}[1]{\textit{Proof. (Sketch)}\xspace #1\hfill\ensuremath{\Box}\smallskip}
\newcommand{\myProof}[1]{\textit{Proof.}\xspace #1\hfill\ensuremath{\Box}\smallskip}

\def\naf{\ensuremath{\raise.17ex\hbox{\ensuremath{\scriptstyle\mathtt{\sim}}}}\xspace}

\begin{document}

\lefttitle{}

\jnlPage{1}{8}
\jnlDoiYr{2021}
\doival{10.1017/xxxxx}

\title[Quantifying over Optimum Answer Sets]{Quantifying over Optimum Answer Sets%
\thanks{Partially supported by MISE under project EI-TWIN n. F/310168/05/X56 CUP B29J24000680005, and MUR under projects: PNRR FAIR - Spoke 9 - WP 9.1 CUP H23C22000860006, Tech4You CUP H23C22000370006, and PRIN PINPOINT CUP H23C22000280006.}}

\begin{authgrp}
\author{\sn{Giuseppe} \gn{Mazzotta}\hspace{.5cm} \sn{Francesco} \gn{Ricca}\hspace{.5cm}\sn{Mirek} \gn{Truszczynski}
}
\affiliation{\hspace{-.6cm} University of Calabria, Italy \hspace{.3cm} University of Calabria, Italy \hspace{.7cm} University of Kentucky}

\end{authgrp}

\history{\sub{xx xx xxxx;} \rev{xx xx xxxx;} \acc{xx xx xxxx}}
\maketitle
\begin{abstract}

Answer Set Programming with Quantifiers (\ASPQ) has been introduced to provide a natural extension of ASP modeling to problems in the polynomial hierarchy (PH). However, \ASPQ lacks a method for encoding in an elegant and compact way problems requiring a polynomial number of calls to an oracle in $\Sigma_n^p$ (that is, problems in $\Delta_{n+1}^p$). Such problems include, in particular, optimization problems.
In this paper we propose an extension of \ASPQ, in which component programs may contain weak constraints. Weak constraints can be used both for expressing local optimization within quantified component programs and for modeling global optimization criteria. We showcase the modeling capabilities of the new formalism through various application scenarios. Further, we study its computational properties obtaining complexity results and unveiling non-obvious characteristics of \ASPQ programs with weak constraints.

    
\end{abstract}
\section{Introduction}
Answer set programming (ASP)~\citep{DBLP:journals/cacm/BrewkaET11,DBLP:journals/ngc/GelfondL91} has been proposed over two decades ago as a variant of logic programming for modeling and solving search and optimization problems \citep{MarekT1999,DBLP:journals/amai/Niemela99}. Today it is among the most heavily studied declarative programming formalisms with highly effective processing tools and an ever-growing array of applications \citep{DBLP:journals/cacm/BrewkaET11,BrewkaET2016}. 
Focusing on decision problems, the scope of applicability of ASP is that of the class $\Sigma_2^P$ \citep{DBLP:journals/csur/DantsinEGV01}. 
This class includes a vast majority of problems of practical interest. However, many important decision problems belong to higher complexity  classes~\citep{schaefer2002completeness,DBLP:journals/tcs/Stockmeyer76}. 
For this reason, many language extensions have been proposed that expand the expressivity of ASP~\citep{DBLP:journals/tplp/BogaertsJT16,DBLP:journals/tplp/FandinnoLRSS21,DBLP:journals/tplp/AmendolaRT19}.
Among these, Answer Set Programming with Quantifiers (\ASPQ)~\citep{DBLP:journals/tplp/AmendolaRT19} has been recently introduced to offer a natural declarative { means} to model problems in the entire Polynomial Hierarchy (PH).

Roughly speaking, the definition of a problem in $\Sigma_n^P$ can be often reformulated as  ``there is an answer set of a program $P_1$ such that for every answer set of a program $P_2$, $\ldots$ there is an answer set of $P_n$, so that a stratified program with constraint $C$, modeling admissibility of a solution, is coherent,'' (and a similar sentence starting with ``for all answer set of program $P_1$'' can be used to encode problems $\Pi_n^P$). 

Both the original paper~\citep{DBLP:journals/tplp/AmendolaRT19} on \ASPQ, and the subsequent one \citep{DBLP:conf/lpnmr/AmendolaCRT22} presented several examples of problems outside the class $\Sigma^P_2$ that allow natural representations as \ASPQ programs. 
Furthermore, \cite{DBLP:conf/lpnmr/AmendolaCRT22} first, and ~\cite{DBLP:journals/tplp/FaberMR23} later, provided efficient tools for evaluating \ASPQ specifications providing empirical evidence of practical potential of \ASPQ.  

However, \ASPQ lacks a convenient method for encoding in an elegant way preference and optimization problems~\citep{schaefer2002completeness,DBLP:journals/tkde/BuccafurriLR00}.

In this paper, we address this issue by proposing an extension of \ASPQ with \textit{weak constraints} { or \ASPQW, in short.} 
Weak constraints were introduced in ASP by \cite{DBLP:journals/tkde/BuccafurriLR00} to define preferences on answer sets. 
They are today a standard { construct of ASP~\citep{DBLP:journals/tplp/CalimeriFGIKKLM20}}, { used} to model problems { in the class} $\Delta^P_3$ (i.e., the class of problems that can be solved by a polynomial  number of calls to a $\Sigma^P_2$ oracle).
In \ASPQW, weak constraints have dual purposes: expressing local optimization within quantified subprograms and modeling global optimization criteria. Both features increase the modeling efficacy of the language, which we demonstrate through example problems. Further, we investigate the computational properties of \ASPQ programs with weak constraints and obtain complexity results that reveal some non-obvious characteristics of the new language. Among these, the key positive result states that \ASPQW programs with $n$ alternating quantifiers can model problems complete for $\Delta_{n+1}^P$.

\section{Answer Set Programming}\label{sec:ASP}
We now recall Answer Set Programming (ASP)~\citep{DBLP:journals/cacm/BrewkaET11,DBLP:journals/ngc/GelfondL91} and introduce the notation employed in this paper. 

\myParagraph{The Syntax of ASP.}
Variables are strings starting with uppercase letters, and constants are non-negative integers or strings starting with lowercase letters. 
A \textit{term} is either a variable or a constant. 
A \textit{standard atom} is an expression of the form $p(t_1, \ldots, t_n)$, where $p$ is a \textit{predicate} of arity $n$ and $t_1, \ldots, t_n$ are terms.
A standard atom $p(t_1, \ldots, t_n)$ is \textit{ground} if $t_1, \ldots, t_n$ are constants.
A \textit{standard literal} is an atom $p$ or its negation $\naf p$.
%
%
%
An \textit{aggregate element} is a pair $t_1,\ldots,t_n : \mathit{conj}$,
where $t_1,\ldots,t_n$ is a non-empty list of terms, and $\mathit{conj}$ is a non-empty conjunction of standard literals.
An \textit{aggregate atom} is an expression $f\{e_1;\ldots;e_n\} \prec T$, where $f \in \{\#count,\#sum\}$ is an \textit{aggregate function symbol}, $\prec\ \in \{<, \leq, >, \geq, =\}$
is a comparison operator, $T$ is a term called the \textit{guard}, and $e_1,\ldots,e_n$ are aggregate elements.
An \textit{atom} is either a standard atom or an aggregate atom.
A \textit{literal} is an atom (positive literal) or its negation (resp. negative literal).
 The \textit{complement} of a literal $l$ is denoted by $\overline{l}$, and it is $\naf a$, if $l = a$, or $a$, if $l = \naf a$, where $a$ is an atom.
For a set of literals $L$, $L^+$ and $L^{-}$ denote the set of positive and negative literals in $L$, respectively.
%
A \textit{rule} is an expression of the form:
\begin{equation}\label{eq1}
h \leftarrow b_1,\ldots,b_k, \naf b_{k+1}, \ldots, \naf b_m.
\end{equation}
where $m\geq k \geq0$. 
Here $h$ is a standard atom or is empty, and all $b_i$ with $i\in[1,m]$ are atoms.
We call $h$ the \textit{head} and $b_1,\ldots,b_k, \naf b_{k+1}, \ldots, \naf b_m$ the \textit{body} of the rule (\ref{eq1}). If the head is empty, the rule is a \textit{hard constraint}. 
If a rule (\ref{eq1}) has a non-empty head and $m=0$, the rule is a \textit{fact}.
Let $r$ be a rule, $h_r$ denotes the head of $r$, and $B_r = B^+_r \cup B^-_r$ where $B^+_r$ (resp. $B^-_r$) is the set of all positive (resp. negative) literals in the body of $r$.

A \textit{weak constraint}~\citep{DBLP:journals/tkde/BuccafurriLR00} is an expression of the form:
\begin{equation}\label{eq2}
\weakarr b_1,\ldots,b_k, \naf b_{k+1}, \ldots, \naf b_m\ [w@l,T],
\end{equation}
where, $m \geq k\geq 0$, $b_1,\ldots,b_k, b_{k+1}, \ldots, b_m$ are standard atoms, $w$ and $l$ are terms, and $T=t_1,\ldots,t_n$ is a tuple of terms with $n\geq 0$.
Given an expression $\epsilon$ (atom, rule, weak constraint, etc.), $\vars{\epsilon}$ denotes the set of variables appearing in $\epsilon$; $at(\epsilon)$ denotes the set of standard atoms appearing in $\epsilon$; and $\preds{\epsilon}$ denotes the set of predicates appearing in $\epsilon$.
For a rule $r$, the \textit{global variables} of $r$ are all those variables appearing in $h_r$ or in some standard literal in $B_r$ or in the guard of some aggregates.
A rule $r$ is \textit{safe} if its global variables appear at least in one positive standard literal in $B_r$, { and each variable appearing into an aggregate element $e$ either is global or appears in some positive literal of $e$}~\citep{DBLP:books/sp/CeriGT90,DBLP:journals/ai/FaberPL11};
a weak constraint $v$ of the form (\ref{eq2}) is safe if $\vars{B_v^-} \subseteq \vars{B_v^+}$ and $\vars{\{w,l\}} \cup \vars{T} \subseteq \vars{B_v^+}$.
A \textit{program} $\prg$ is a set of safe rules and safe weak constraints. 
Given a program $\prg$, $\rules{\prg}$ and $\weak{\prg}$ denote the set of rules and weak constraints in $\prg$, respectively, and $\heads{\prg}$ denotes the set of atoms appearing as heads of rules in $\prg$.

A choice rule~\citep{DBLP:journals/ai/SimonsNS02} is an expression of the form:
$\{e_1;\ldots;e_k\}\leftarrow l_1,\ldots,l_n,$
where each choice element $e_i$ is of the form $a^i:b^i_1,\ldots,b^i_{m_i}$, where $a^i$ is a standard atom, $m_i \geq 0$, and $b^i_1,\ldots,b^i_{m_i}$ is a conjunction of standard literals.
For simplicity, choice rules can be seen as a shorthand for certain sets of rules. 
In particular, each choice element $e_i$ corresponds to: 
$a^i\leftarrow b^i_1,\ldots,b^i_{m_i},l_1,\ldots,l_n,\ \naf na^i$,
$na^i\leftarrow b^i_1,\ldots,b^i_{m_i},l_1,\ldots,l_n,\ \naf a^i$
where $na^i$ denotes the standard atom obtained from $a^i$ by substituting the predicate of $a$, say $p$, with a fresh predicate $p^{\prime}$ not appearing anywhere else in the program.

\myParagraph{The Semantics of ASP.}\label{sec:semantics}
Assume a program $\prg$ is given. The \textit{Herbrand Universe} is the set of all constants appearing in $\prg$ (or a singleton set consisting of any constant, if no constants appear in $\prg$) and is denoted by $\HU$;
whereas the \textit{Herbrand Base}, that is the set of possible ground standard atoms obtained from predicates in $\prg$ and constants in $\HU$, is denoted by $\HB$.
Moreover, $ground(\prg)$ denotes the set of possible ground rules obtained from rules in $\prg$ by proper variable substitution with constants in $\HU$. 
{An \textit{interpretation} $I \subseteq \HB$ is a set of standard atoms. 
A ground standard literal $l = a$ (resp. $l=\naf a$) is true w.r.t. $I$ if $a \in I$ (resp. $a \notin I$), otherwise it is false. 
A conjunction $conj$ of literals is true w.r.t. $I$ if every literal in $conj$ is true w.r.t. $I$, otherwise it is false. 
Given a ground set of aggregate elements $S = \{e_1;\ldots;e_n\}$, $eval(S,I)$ denotes the set of tuples of the form $(t_1,\ldots,t_m)$ such that there exists an aggregate element $e_i\in S$ of the form $t_1,\ldots,t_m: conj$ and $conj$ is true w.r.t. $I$; $I(S)$, instead, denotes the multi-set $[t_1\mid (t_1,\ldots,t_m) \in eval(S,I)]$. 
A ground aggregate literal of the form $f\{e_1;\ldots;e_n\}\succ t$ (resp. $\naf\ f\{e_1;\ldots;e_n\}\succ t$) is true w.r.t. $I$ if $f(I(\{e_1,\ldots,e_n\}))\succ t$ holds (resp. does not hold); otherwise it is false. 
An interpretation $I$ is a \textit{model} of $\prg$ iff for each rule $r \in ground(\prg)$ either the head of $r$ is true w.r.t. $I$ or the body of $r$ is false w.r.t. $I$. Given an interpretation $I$, $\prg^I$ denotes the \textit{FLP-reduct} (cfr.~\cite{DBLP:journals/ai/FaberPL11}) obtained by removing all those rules in $\prg$ having their body false, and removing negative literals from the body of remaining rules. A model $I$ of $\prg$ is also an \textit{answer set} of $\prg$ if for each $I'\subset I$, $I'$ is not a model of $\prg^I$.}
We write $\AS{\prg}$ for the set of answer sets of $\prg$. 
A program $\prg$ is \textit{coherent} if it has at least one answer set (i.e. $\AS{\prg} \neq \emptyset$); otherwise, $\prg$ is \textit{incoherent}.
For a program $\prg$ and an interpretation $I$, let the set of weak constraint violations be 
$\weights(P,I) = \{(w,l,T) \mid \ \weakarr b_1,\ldots,b_m\ [w@l,T] \in ground(\weak{\prg}),$
$b_1,\ldots,b_m$ are true w.r.t. $I$, $w$ and $l$ are integers and $T$ is a tuple of ground terms$\}$, then the cost function of $\prg$ is %
$\levelCost{\prg}{I}{l} = \Sigma_{(w,l,T) \in \weights(P,I)} w,$ for every integer $l$.
%
%
Given a program $\prg$ and two interpretations $I_1$ and $I_2$, we say that that $I_1$ is \textit{dominated} by $I_2$ if there is an integer $l$ such that $\levelCost{\prg}{I_2}{l}<\levelCost{\prg}{I_1}{l}$ and for all integers $l^{\prime} > l$, $\levelCost{\prg}{I_2}{l^{\prime}} = \levelCost{\prg}{I_1}{l^{\prime}}$. An answer set $M\in \AS{\prg}$ is an \textit{optimal} answer set if it is not dominated by any $M^{\prime} \in \AS{\prg}$. 
Intuitively, optimality amounts to minimizing the weight at the highest possible level, with each level used for tie breaking for the level directly above.
The set $\OptAS{\prg}\subseteq\AS{\prg}$ denotes the set of optimal answer sets of $\prg$.
\section{Quantified Answer Set Programming with Weak Constraints}\label{preliminaries:section:qasp}
In this section, we introduce an extension of Answer Set Programming with Quantifiers (\ASPQ) \citep{DBLP:journals/tplp/AmendolaRT19} that explicitly supports weak constraints~\citep{DBLP:journals/tkde/BuccafurriLR00} for modeling optimization problems. 

It is worth noting that \ASPQ can be used to model problems with model preferences and optimization criteria; however, this comes at the price of non-elegant and somehow redundant modeling. 
For this reason, in analogy to what has been done for ASP, it makes sense to contemplate weak constraints in \ASPQ.

A \textit{quantified ASP program with weak constraints} (\ASPQW program) $\qprg$ is of the form:
\begin{equation}
\Box_1 \prg_1\ \Box_2 \prg_2\ \cdots\ \Box_n \prg_n :  \cprg : \wprg,
\label{eq:aspqwprogram}
\end{equation}
\noindent where, for each $i=1,\ldots,n$, $\Box_i \in \{ \exists^{st}, \forall^{st}\}$, $\prg_i$ is an ASP program possibly with weak constraints, 
$\cprg$ is a (possibly empty) \textit{stratified} program~\citep{DBLP:books/sp/CeriGT90} with constraints, and $\wprg$ is a (possibly empty) set of weak constraints such that $B_{\wprg} \subseteq B_{\prg_1}$. 
The number of quantifiers in $\qprg$ is denoted by $nQuant(\qprg)$. 

As it was in the base language, \ASPQW programs are quantified sequences of subprograms ending with a \textit{constraint program} $\cprg$. Differently from \ASPQ, in \ASPQW weak constraints are allowed in the subprograms $\prg_i$ ($1 \leq i \leq n)$, that is, quantification is over optimal answer sets. Moreover, the \textit{global weak constraints} subprogram $\wprg$ is introduced to specify (global) optimality criteria on quantified answer sets. 

Formally, the \textit{coherence} of \ASPQW\ programs is defined as follows:
\begin{itemize}
\item $\exists^{st} \prg:\cprg:\wprg$ is coherent, if there exists $M\in \OptAS{\prg}$ such that $\cprg \cup \fix{\prg}{M}$ admits an answer set;
\item $\forall^{st} \prg:\cprg:\wprg$ is coherent, if for every $M\in \OptAS{P}$, 
$C\cup \fix{\prg}{M}$ admits an answer set;
\item $\exists^{st} \prg\ \qprg$ is coherent, if there exists $M\in \OptAS{P}$ such 
that $\fixedqprg{\prg}{M}$ is coherent;
\item $\forall^{st} \prg\ \qprg$ is coherent, if for every $M\in \OptAS{\prg}$, $\fixedqprg{\prg}{M}$
is coherent.
\end{itemize}
{
where $\fix{\prg}{M}$ denotes the set of facts and constraints $\{ a \mid a\in M \cap \HB\} \cup \{ \leftarrow a \mid a\in \HB \setminus M\}$,
and $\fixedqprg{\prg}{M}$ denotes the \ASPQW\ program of the form~(\ref{eq:aspqwprogram}), where $P_1$ is 
replaced by $P_1\cup \fix{\prg}{M}$, that is, 
$\fixedqprg{\prg}{M} =\Box_1 (P_1\cup \fix{\prg}{M})\ \Box_2 P_2\  \cdots \Box_n P_n :  C:\wprg.$
}

For an existential \ASPQW\ program $\qprg$, $M \in \OptAS{\prg_1}$ is a \textit{quantified answer set} of $\qprg$, if $((\Box_2 \prg_2 \cdots \Box_n \prg_n :  \cprg):C^w)_{\prg_1,M}$ is coherent.
We denote by $QAS(\qprg)$ the set of all quantified answer sets of $\qprg$.

To illustrate the definitions above, let us consider the following \ASPQW\ program
$\Pi = \exists^{st} \prg_1 \forall^{st} \prg_2 \cdots \exists^{st} \prg_{n-1} \forall^{st} \prg_n: \cprg: \wprg$.
``Unwinding'' the definition of coherence yields that 
$\qprg$ is coherent if there exists an \textit{optimal} answer set $M_1$ of $\prg_1'$ 
such that for every \textit{optimal} answer set $M_2$ of $\prg_2'$ 
there exists an \textit{optimal} answer set $M_3$ of $\prg_3'$, and so on until 
there exists an \textit{optimal} answer set $M_{n-1}$ of $\prg_{n-1}'$ 
such that for every \textit{optimal} answer set $M_n$ of $\prg_n'$, 
there exists an answer set of $\cprg \cup \fix{\prg_n'}{M_n}$, 
where $\prg_1'=\prg_1$, and $\prg_i'=\prg_i\cup \fix{\prg_{i-1}'}{M_{i-1}}$ with $i\geq2$. 
Note that, as in \ASPQ, the constraint program $\cprg$ has the role of selecting admissible solutions. Weak constraints could be allowed in $\cprg$, but they would be redundant. 
Indeed, $\cprg$, being stratified with constraints, admits at most one answer set, which would necessarily be optimal. 
%
In contrast, the \textit{local weak constraints} (possibly) occurring in subprograms $P_i$ are essential for determining coherence. 

\begin{example}[Impact of local weak constraints]
Let $\Pi_1 = \exists P_1 \forall P_2 : C $, and $\Pi_2 = \exists Q_1 \forall Q_2 : C$, where $C = \{\leftarrow d,f\}$ and also:
    
    {
        \footnotesize
        \begin{minipage}{.24\textwidth}
            \[
                \prg_1 = \left\{\begin{array}{l}
                     \{a;b\}=1  \leftarrow  \\
                     \{c;d\}=1  \leftarrow  \\
                        \weakarr c\ [1@1]  \\
                \end{array}\right\}
            \]
        \end{minipage}
        \begin{minipage}{.26\textwidth}
            \[
                \prg_2 = \left\{\begin{array}{l}
                     \{e,f\}  \leftarrow \\
                     \leftarrow \naf e,\naf f  \\
                     \weakarr e,f\ [1@1]  \\
                \end{array}\right\}
            \]
        \end{minipage}
        \begin{minipage}{.24\textwidth}
            \[
                Q_1 = \left\{\begin{array}{l}
                     \{a;b\}=1  \leftarrow  \\
                     \{c;d\}=1  \leftarrow  \\
                \end{array}\right\}
            \]
        \end{minipage}
        \begin{minipage}{.2\textwidth}
            \[
                Q_2 = \left\{\begin{array}{l}
                     \{e,f\}  \leftarrow \\
                     \leftarrow \naf e,\naf f  \\
                \end{array}\right\}
            \]
        \end{minipage} \smallskip
    }
    
\noindent Note that, $\Pi_2$ can be obtained from $\Pi_1$ by discarding weak constraints.
First, we observe that $\Pi_1$ is incoherent. Indeed, the optimal answer sets of $P_1$ are $\OptAS{P_1} = \{\{a,d\},\{b,d\}\}$. By applying the definition of coherence, when we consider $M = \{a,d\}$, we have that $\OptAS{P_2^{\prime}} = \{\{e,a,d\},\{f,a,d\} \}$. Once we set  $M^{\prime} = \{f,a,d\}$, the program $C^{\prime}$ is not coherent, and so $M = \{a,d\}$ is not a quantified answer set.
Analogously, when we consider the second answer set of $P_1$, i.e., $M = \{b,d\}$, we have that $\OptAS{P_2^{\prime}} =\{\{e,a,d\},\{f,a,d\} \}$. But, when we set $M^{\prime} = \{f,b,d\}$, the program $C^{\prime}$ is not coherent. Thus, $\Pi_1$ is incoherent.
On the contrary, $\Pi_2$ is coherent. Indeed, $\AS{Q_1} = \{\{a,d\},\{b,d\},\{a,c\},\{b,c\}\} = \OptAS{P_1} \cup \{\{a,c\},\{b,c\}\}$. 
The first two, we know, do not lead to a quantified answer set. But, when we set $M = \{a,c\}$, since $d$ is false, it happens that $C^{\prime}$ is coherent (e.g., when we consider the answer set $\{e,a,c\}$ of $Q_2^{\prime}$). 
Thus, local weak constraints can affect coherence by discarding not optimal candidates.
\end{example}

Global weak constraints in $\wprg$ do not affect coherence, but they serve to define  optimality criteria across quantified answer sets. For this reason, we require that $\wprg$ is defined over the same Herbrand base of $\prg_1$. 
Furthermore, note that $\wprg$ plays no role in universal \ASPQW programs, where coherence is the sole meaningful task.

Given an existential \ASPQW\ program $\qprg$ and two quantified answer sets $Q_1, Q_2 \in QAS(\qprg)$, we say that $Q_1$ is dominated by $Q_2$ if there exists an integer $l$ such that $\levelCost{\prg_1^*}{Q_2}{l}<\levelCost{\prg_1^*}{Q_1}{l}$ and for every integer $l^{\prime} > l$, $\levelCost{\prg_1^*}{Q_2}{l} = \levelCost{\prg_1^*}{Q_1}{l}$, where $\prg_1^* = \prg_1 \cup \wprg$. 
An \textit{optimal quantified answer set} is a quantified answer set $Q \in QAS(\qprg)$ that is not dominated by any $Q^{\prime} \in QAS(\qprg)$. 

\begin{example}[Optimal quantified answer sets]
    Let $\Pi = \exists P_1 \forall P_2 : C : {\wprg}$ be such that:

    {
        \footnotesize
        \begin{minipage}{.22\textwidth}
            \[
                \prg_1 = \left\{
                \{a;b;c\}\leftarrow
                \right\}
            \]
        \end{minipage}
        \begin{minipage}{.22\textwidth}
            \[
                \prg_2 = \left\{
                \begin{array}{l}
                \{a';b';c'\}\leftarrow\\
                \leftarrow a', \naf b'\\
                \leftarrow \naf a', \naf b'\\
                \leftarrow a', \naf c'\\
                \leftarrow \naf a', \naf c'
                \end{array}
                \right\}
            \]
        \end{minipage}
        \begin{minipage}{.22\textwidth}
            \[
                \cprg = \left\{\begin{array}{l}
                     \leftarrow a,\naf a'  \\
                     \leftarrow b,\naf b'  \\
                     \leftarrow c,\naf c'  \\
                \end{array}\right\}
            \]
        \end{minipage}
        \begin{minipage}{.22\textwidth}
            \[
                \wprg = \left\{\begin{array}{l}
                     \leftarrow \naf a\ [1@1,a]  \\
                     \leftarrow \naf b\ [1@1,b]  \\
                     \leftarrow \naf c\ [1@1,c]  \\
                \end{array}\right\}
            \]
        \end{minipage}
        \smallskip
    }
    
Given that $QAS(\Pi) = \{\{\},\{b\},\{c\},\{b,c\}\}$, 
we have that: the cost of $\{\}$ is 3, since it violates all weak constraints in $C^{w}$; \{b\} and \{c\} cost 2, since $\{b\}$ (resp. $\{c\}$) violates the first and the third (resp. second) weak constraint; and, \{b,c\} costs 1, because it only violates the first weak constraint.
Thus, the optimal quantified answer set of $\Pi$ is $\{b,c\}$. 
\end{example}

{

Let $\qprg = \Box_1 \prg_1 \ldots \Box_n P_n$ be an \ASPQW program. $\qprg$ satisfies the stratified definition assumption if for each $1\leq i\leq n$, $\heads{\prg_i}\cap at(\prg_j) = \emptyset$, with $1\leq j < i$. 
In what follows, we assume w.l.o.g. that \ASPQW programs satisfy the \textit{stratified definition assumption}. 

It is worth noting that, standard \ASPQ allows for the specification of preferences and optimization. The basic pattern for obtaining optimal models in \ASPQ is to ``clone'' a program and use an additional quantifier over its answer sets. This allows us to compare pairs of answer sets and, by means of a final constraint program, to select optimal ones. 
For example, assume program $P_1$ models the candidate solutions of a problem and, for the sake of illustration, that we are interested in those minimizing the number of atoms of the form $a(X)$. This desideratum can be modeled directly in standard ASP by adding a weak constraint $\leftarrow_w a(X) [1@1,X]$. 
On the other hand, in \ASPQ we can model it with the program $\exists P_1 \forall P_2 : C$ such that $P_2=clone^s(P_1)$, and $C = \{ \leftarrow \#count\{X: a(X)\}=K, \#count\{X:a^s(X)\}<K\}$. 
Here, we are comparing the answer sets of $P_1$ with all their ``clones'', and keep those that contain a smaller (or equal) number of atoms of the form $a(X)$. This pattern is easy to apply, but it is redundant; also note that checking coherence of an \ASPQ program with two quantifiers is in $\Sigma_2^p$~\citep{DBLP:journals/tplp/AmendolaRT19}, whereas optimal answer set checking of a program with weak constraints is in $\Delta_{2}^p$~\citep{DBLP:journals/tkde/BuccafurriLR00}.
These observations motivate the introduction of weak constraints in \ASPQ, which will be further strengthened in the following sections.
}

\section{Modeling examples}\label{sec:modeling}
We showcase the modeling capabilities of \ASPQW by considering two example scenarios where both global and local weak constraints play a role: the Minmax Clique problem~\citep{cao1995minimax}, and Logic-Based Abduction~\citep{DBLP:journals/jacm/EiterG95}.

\myParagraph{Minmax clique problem.}
Minimax problems are prevalent across numerous research domains.
Here, we focus on the Minmax Clique problem, as defined by \cite{ko1995complexity}, although other minimax variants can be also modeled.

Given a graph $G = \langle V,E \rangle$, let $I$ and $J$ be two finite sets of indices,
and $(A_{i,j})_{i\in I,j\in J}$ a partition of $V$. We write $J^I$ for the
set of all total functions from $I$ to $J$.  
For every total function $f\colon I\rightarrow J$ we denote by $G_f$ the 
subgraph of $G$ induced by $\bigcup_{i\in I} A_{i,f(i)}$. 
The {\sc Minmax Clique} optimization problem is defined as follows: 
Given a graph $G$, sets of indices $I$ and $J$, a partition $(A_{i,j})_{i\in I,j\in J}$, find the integer $k$ ($k \leq |V|)$, such that 
\[
k = \min_{f\in J^I}\;\max\{|Q|: \mbox{$Q$ is a clique of $G_f$}\}.
\]
%
%
%
%
%
The following program of the form $\qprg = \exists \prg_1 \exists \prg_2 : \cprg : \wprg$, encodes the problem:

{\footnotesize
\[
\prg_1 = \left
        \{\begin{array}{rll}
            v(i,j,a)                     & \leftarrow & \forall i \in I, j\in J, a \in A_{i,j}\\
            inI(i)                       & \leftarrow & \forall i \in I\\
            inJ(j)                       & \leftarrow & \forall j \in J\\
            e(x,y)                    & \leftarrow & \forall (x,y) \in E\\
            \{f(i,j):inJ(j)\}=1          & \leftarrow inI(i)    & \\
            \{valK(1);\ldots;valK(\vert V \vert)\} = 1               & \leftarrow &
        \end{array}\right\}
\]
\[
\prg_2 = \left
        \{\begin{array}{rll}
            n_f(X) &\leftarrow & f(I,J),v(I,J,X) \\
            e_f(X,Y) &\leftarrow & n_f(X),n_f(Y),e(X,Y) \\
            \{inClique(X):n_f(X)\}&\leftarrow&\\
            &\leftarrow &inClique(X),inClique(Y),X<Y,\naf e_f(X,Y)\\
            &\weakarr & n_f(X),\naf inClique(X)\ \ [1@1,X]
        \end{array}\right\}
\]
\begin{minipage}{.49\textwidth}
\[
\cprg = \left
        \{\begin{array}{rl}
            \leftarrow & valK(K),\#count\{X:inClique(X)\}\neq K
        \end{array}\right\}
\]    
\end{minipage}
\hfill
\begin{minipage}{.49\textwidth}
\[
\wprg = \left
        \{\begin{array}{l}
            \weakarr val(K)\ \ [K@1]
        \end{array}\right\}
\]
\end{minipage}
\smallskip
}

The input is modeled in program $\prg_1$ as follows: Node partitions are encoded as facts of the form $v(i,j,x)$ denoting that a node $x$ belongs to the partition $a \in A_{i,j}$; facts of the form $inI(i)$ and $inJ(j)$ model indexes $i\in I$ and $j \in J$, respectively; and, the set of edges $E$ is encoded as facts of the form $e(x,y)$ denoting that the edge $(x,y) \in E$. 
The first choice rule in $\prg_1$ guesses one total function $f : I \rightarrow J$, which is encoded by binary predicate $f(i,j)$ denoting that the guessed function maps $i$ to $j$.
The second choice rule guesses one possible value for $k$, modeled by predicate $valK(x)$. 
Thus, there is an answer set of program $\prg_1$ for each total function $f$ and a possible value for $k$.
 
Given an answer set of $\prg_1$, program $\prg_2$ computes the maximum clique of the subgraph of $G$ induced by $f$, i.e., $G_f$. 
To this end, the first rule computes the nodes of $G_f$ in predicate $n_f(X)$, by joining  predicate $f(I,J)$ and $v(I,J,X)$. The second rule computes the edges of $G_f$ considering the edges of $G$ that connect nodes in $G_f$. 
The largest clique in $G_t$ is computed by a $(i)$ choice rule that guesses a set of nodes (in predicate $inClique$), $(ii)$ a constraint requiring that nodes are mutually connected, and $(iii)$ a weak constraint that minimizes the number of nodes that are not part of the clique.
At this point, the program $\cprg$ verifies that the size of the largest clique in the answer set of $\prg_2$ is exactly the value for $k$ in the current answer set of $\prg_1$. Thus, each quantified answer set of $\qprg$ models a function $f$, such that the largest clique of induced graph $G_f$ has size $k$. Now, the global weak constraints in $\wprg$ prefer the ones that give the smallest value of $k$. 

The decision version of this problem is $\Pi_2^p$-complete~\citep{ko1995complexity}. Thus, a solution to the {\sc Minmax Clique} can be computed by a logarithmic number of calls to an oracle in $\Pi_2^p$, so the problem belongs to $\Theta_3^P$~\citep{DBLP:journals/siamcomp/Wagner90}. It is (somehow) surprising that we could write a natural encoding without alternating quantifiers (indeed, $\qprg$ features two \textit{existential} quantifiers). This phenomenon is more general. We will return to it
in Section~\ref{sec:complexity}.


\myParagraph{Logic-Based Abduction}
Abduction plays a prominent role in Artificial Intelligence 
as an essential common-sense reasoning mechanism
~\citep{DBLP:journals/ai/Morgan71,DBLP:conf/ijcai/Pople73}. 

In this paper we focus on the Propositional Abduction Problem (PAP) \citep{DBLP:journals/jacm/EiterG95}. 
The PAP is defined as a tuple of the form $\mathcal{A}=\langle V,T,H,M \rangle$, where $V$ is a set of variables, $T$ is a consistent propositional logic theory over variables in $V$, $H \subseteq V$ is a set of hypotheses, and $M \subseteq V$ is a set of manifestations.
A \textit{solution} to the PAP problem $\cal{A}$ is a set $S\subseteq H$ such that $T\cup S$ is consistent and $T \cup S \vDash M$. 
Solutions to $\cal{A}$, denoted by $sol(\cal{A})$, can be ordered by means of some preference relation $<$. The set of optimal solutions to $\cal{A}$ is defined as $sol_{<}(\cal{A}) =$ $\{S \in sol(\cal{A}) \mid \nexists$ $S'\in sol(\cal{A}) \textit{ such that } \vert S'\vert < \vert S\vert\}$. 
A hypothesis $h \in H$ is \textit{relevant} if $h$ appears at least in one solution $S \in sol_{<}(\cal{A})$. 
%
The main reasoning tasks for PAP are beyond NP~\citep{DBLP:journals/jacm/EiterG95}.

In the following, we assume w.l.o.g. that the theory $T$ is a boolean 3-CNF formula over variables in $V$. 
Recall that, a 3-CNF formula is a conjunction of clauses $C_1 \wedge \ldots \wedge C_n$, where each clause is of the form $C_i = l_i^1 \vee l_i^2 \vee l_i^3$, and each literal $l_i^j$ (with $1 \leq j \leq 3$) is either a variable $a \in V$ or its (classical) negation $\neg a$.

Given a PAP problem $\mathcal{A}=\langle V,T,H,M \rangle$ we aim at computing a solution $S \in sol_{<}(\cal{A})$.
To this end, we use an \ASPQW program of the form $\exists \prg_1 \forall \prg_2: \cprg: \wprg$, where:

{
\footnotesize
\[
\prg_1 = \left
        \{\begin{array}{rlll}
            v(x)            &\leftarrow & &\forall\ x \in V \\
            lit(C_i,a,t)    &\leftarrow & &\forall\ a \in V \mid a \in C_i\\
            lit(C_i,a,f)    &\leftarrow & &\forall\ a \in V \mid \naf a \in C_i \\
            h(x)            &\leftarrow & &\forall\ x \in H \\
            m(x)            &\leftarrow & &\forall\ x \in M \\
            cl(X)           &\leftarrow &  lit(X,\_,\_) &\\
            \{s(X):h(X)\}&\leftarrow & &\\
            \{tau(X,t);tau(X,f)\}=1 &\leftarrow& v(X)& \\
            satCl(C) &\leftarrow& lit(C,A,V), tau(A,V)& \\
            &\leftarrow& cl(C),\naf satCl(C)& \\
            &\leftarrow& s(X), tau(X,f)& \\
            
        \end{array}\right\}
\]

\[
\prg_2 = \left
        \{\begin{array}{rll}
            \{tau'(X,t);tau'(X,f)\}=1 &\leftarrow &v(X) \\

            satCl'(C) &\leftarrow &lit(C,A,V), tau'(A,V) \\
            unsatTS'  &\leftarrow &cl(C),\naf satCl'(C) \\
            unsatTS'  &\leftarrow &s(X),\naf tau'(X,f) \\
        \end{array}\right\}
\]
\begin{minipage}{.49\textwidth}
\[
\cprg = \left
        \{\begin{array}{c}
            \leftarrow \naf unsatTS', m(X), tau'(X,f)
        \end{array}\right\}
\]    
\end{minipage}
\hfill
\begin{minipage}{.49\textwidth}

\[
\wprg = \left
        \{\begin{array}{l}
            \weakarr s(X) \ [1@1,X]
        \end{array}\right\}
\]    
\end{minipage}
\smallskip
}

The aim of $\prg_1$ is to compute a candidate solution $S \subseteq H$ such that $T \cup S$ is consistent; $\prg_2$ and $\cprg$ ensure that $T \cup S \vDash M$, and $\wprg$ ensures that $S$ is cardinality minimal. 
More in detail, in program $P_1$, the variables $V$, hypothesis $H$, and manifestations $M$, are encoded by means of facts of the unary predicates $v$, $h$, and $m$, respectively. The formula $T$ is encoded by facts of the form $lit(C,x,t)$ (resp. $lit(C,x,f)$) denoting that { a} variable $x$ occurs in a positive (resp. negative) literal in clause $C$. 
Then, to ease the presentation, we compute in a unary predicate $cl$ the set of clauses.
The first choice rule guesses a solution (a subset of $H$), and the last five rules verify the existence of a truth assignment $\tau$ for variables in $V$, encoded with atoms of the form $tau(x,t)$ (resp. $tau(x,f)$) denoting { that} a variable $x$ is true (resp. false), such that $unsatTS$ is not derived (last constraint). Note that, $unsatTS$ is derived either if a clause is not satisfied or if a hypothesis is not part of the assignment. Thus, the assignment $\tau$ satisfies $T \cup S$, i.e., $T \cup S$ is consistent.
It follows that the answer sets of $\prg_1$ correspond to candidate solutions $S \subseteq H$ such that $T \cup S$ is consistent. 
Given a candidate solution, program $\prg_2$ has one answer set for each truth assignment $\tau'$ that satisfies $T \cup S$, and the program $C$ checks that all such $\tau'$ satisfy also the manifestations in $M$. Thus, every $M \in QAS(\qprg)$ encodes a solution $S \in sol(\cal{A})$. 
The weak constraint in $\wprg$ ensures we single out cardinality minimal solutions by minimizing the extension of predicate $s$.
%
Finally,
let $h$ be a hypothesis, we aim at checking that $h$ is relevant, i.e., $h \in S$ s.t. $S \in sol_{<}(\cal A)$. We solve this task by taking the program $\qprg$ above that computes an optimal solution and adding to $\wprg$ an additional (ground) weak constraint, namely $\weakarr \naf s(h)\ \ [1@0]$. Intuitively, optimal solutions not containing $h$ violate the weak constraint, so if any optimal answer set contains $s(h)$ then $h$ is relevant.


\myParagraph{Remark.}
Checking that a solution to a PAP is minimal belongs to $\Pi_2^p$~\citep{DBLP:journals/jacm/EiterG95}, so the task we have considered so far is complete for $\Theta^P_3$~\citep{DBLP:journals/siamcomp/Wagner90}.
The programs above feature only two quantifiers, whereas alternative encodings in \ASPQ (i.e., without weak constraints) would have required more.
Moreover, we observe that the programs above are rather natural renderings of the definition of the problems that showcase the benefit of modeling optimization in subprograms and at the global level.

\section{Rewriting into plain \ASPQ}\label{sec:rewrite_to_aspq}

In this section, we describe a mapping that transforms an \ASPQW program $\qprg$ into a plain (i.e., without weak constraints) quantifier-alternating \ASPQ program $\qprg'$ that is coherent iff $\qprg$ is coherent.  
This transformation is crucial for enabling the study of the complexity of the primary reasoning tasks of \ASPQW. 
Additionally, it could be applied in an implementation that extends current solvers such as that by Faber et al. (\citeyear{DBLP:journals/tplp/FaberMR23}).

The transformation works by calling a number of intermediate rewritings until none of them can be applied anymore. They
$(i)$ absorb consecutive quantifiers of the same kind; and,
$(ii)$ eliminate weak constraints from a subprogram by encoding the optimality check in the subsequent subprograms.
We first introduce some useful definitions. 
Given program $\qprg$ of the form~{(\ref{eq:aspqwprogram})} we say that two consecutive subprograms $P_i$ and $P_{i+1}$ are \textit{alternating} if $\Box_i \neq \Box_{i+1}$, and are \textit{uniform} otherwise. 
A program $\qprg$ is \textit{quantifier-alternating} if $\Box_i \neq \Box_{i+1}$ for $1 \leq i < n$. 
A subprogram $P_i$ is \textit{plain} if it contains no weak constraint $\weak{P_i} = \emptyset$, and $\qprg$ is \textit{plain} if both all $P_i$ are plain, and $\wprg = \emptyset$. In the following, we assume that $\qprg$ is an \ASPQW program of the form (\ref{eq:aspqwprogram}).


\myParagraph{Rewriting uniform plain subprograms.}
Two plain uniform subprograms can be absorbed in a single equi-coherent subprogram by the transformation $col_1(\cdot)$ defined as follows.

\lemmacolone

Intuitively, if the first two subprograms of $\qprg$ are uniform and plain then $\qprg$ can be reformulated into an equi-coherent {(i.e. $\qprg$ is coherent iff $col_1(\qprg)$ is coherent)} program with one fewer quantifier.

\myParagraph{Rewriting uniform notplain-plain subprograms.}
Next transformations apply to pairs of uniform subprograms $P_1,P_2$ such that $P_1$ is not plain and $P_2$ is plain.
To this end, we first define the $or(\cdot,\cdot)$ transformation. 
    Let $\prg$ be an ASP program, and $l$ be a fresh atom not appearing in $P$, then $or(P,l) = \{H_r\leftarrow B_r,\naf l \mid r \in P\}$.

\obsOr
Intuitively, if the fact $l\leftarrow$ is added to $or(\prg,l)$ then the interpretation $I=\{l\}$ trivially satisfies all the rules and is minimal, thus it is an answer set. On the other hand, if we add the constraint $\leftarrow l$, requiring that $l$ is false in any answer set, then the resulting program behaves precisely as $\prg$ since literal $\naf l$ is trivially true in all the rule bodies.

We are now ready to introduce the next rewriting function $col_2(\cdot)$. {This transformation allows to absorbe a plain existential subprogram into a non plain existential one, thus reducing by one the number of quantifiers of the input \ASPQW program.} 

\defcoltwo

    

\lemmaColTwo

A similar procedure is introduced for the universal case.

\defColThree
\lemmaColThree


Roughly, if the first two subprograms of $\qprg$ are uniform, $P_1$ is not plain, $P_2$ is plain, and the remainder of the program is alternating, then $\qprg$ can be reformulated into an equi-coherent program with one fewer quantifier.

\myParagraph{Rewrite subprograms with weak constraints.}
The next transformations have the role of eliminating weak constraints from a subprogram by encoding the optimality check in the subsequent subprograms.
To this end, we define the $check(\cdot)$ transformation that is useful for simulating the cost comparison of two answer sets of an ASP program $P$.

First, let $\epsilon$ be an ASP expression and $s$ an alphanumeric string.  We define $clone^s(\epsilon)$ as the expression obtained by substituting all occurrences of each predicate $p$ in $\epsilon$ with $p^s$ which is a fresh predicate $p^s$ of the same arity.

\newcommand{\near}{\hspace*{-0.3cm}}

\defTranslateWeak
    
%

Thus, the first two rules compute in predicate $cl_{\prg}$ the cost of an answer set of $P$ w.r.t. his weak constraints, and the following two rules do the same for $clone^o(P)$.
Then, the last four rules derive $dom_{\prg}$ for each answer set of $P$ that is dominated by $clone^o(P)$. 
%

{ We now introduce how to translate away weak constraints from a subprogram.}
\medskip\medskip
\defColFour
\lemmaColFour
    
    
Intuitively, for a pair $M_1,M_2 \in \AS{\prg_1}$, $M_1$ is dominated by $M_2$ if and only if $check(\prg)\cup\fix{\prg}{M_1}\cup clone^o(\fix{\prg}{M_2})$ admits an answer set $M$ such that $dom_{\prg} \in M$. Thus, coherence is preserved since $dom_{\prg_1}$ discards not optimal candidates such as $M_1$.

{ A similar procedure can be defined for universal subprogram.}

\defColFive
\lemmaColFive
    
    

\myParagraph{Translate \ASPQW to \ASPQ.}
\begin{algorithm}[t!]\scriptsize
\caption{Rewriting from \ASPQW to \ASPQ}
    \label{alg:cancel_weak}
    \SetKwInOut{Input}{Input}
    \SetKwInOut{Output}{Output}
    \SetKwRepeat{Do}{do}{while}
    \Input{An \ASPQW program $\Pi$}
    \Output{A quantifier-alternating \ASPQ program}
    \Begin{
        $s$ := $0$; ~~ $\qprg_0$ := $\qprg$\\
        \Do{$stop \neq \top$}{
            $stop$ := $\top$\\
            \ForAll{$ProgramType \in [1,5]$}{
                Let $i \in [1,n]$ be the largest index such that $\qprg_s^{>i}$ is of the type $ProgramType$ \\ 
                \If{ $i \neq \bot$ }{ 
                    $\qprg_{s+1}$ := $replace(\qprg_s,i,col_{ProgramType}(\qprg_s^{\geq i}))$\\
                    $s$ := $s+1$; \\
                    $stop$ := $\bot$\\
                    $break$  \hfill // go to line 12
                }
            }
        }
        \Return{removeGlobal($\qprg_s$)}
    }
\end{algorithm}
Algorithm~\ref{alg:cancel_weak} defines a procedure for rewriting an \ASPQW program $\qprg$ into an \ASPQ program $\qprg'$, made of at most $n+1$ alternating quantifiers, such that $\qprg$ is coherent if and only if $\qprg'$ is coherent. 
In Algorithm~\ref{alg:cancel_weak}, we make use of some (sub)procedures and dedicated notation.
In detail, for a program $\qprg$ of the form (\ref{eq:aspqwprogram}), $\qprg^{\geq i}$ denotes the \textit{i-th suffix program} $\Box_i \prg_i \ldots \Box_n \prg_n : C$, with $1\leq i \leq n$ (i.e., the one obtained from $\qprg$ removing the first $i-1$ quantifiers and subprograms). 
Moreover, the procedure $removeGlobal(\qprg)$ builds an \ASPQ program from a plain one in input (roughly, it removes the global constraint program $\wprg$).
Given two programs $\qprg_1$ and $\qprg_2$, $replace(\qprg_1,i,\qprg_2)$ returns the \ASPQW program obtained from $\qprg_1$ by replacing program $\qprg_1^{\geq i}$ by $\qprg_2$, 
for example $replace(\exists P_1 \forall P2 \exists P_3 : C,\ 2,\ \exists P_4 :C)$ returns $\exists P_1 \exists P_4 :C$. 

In order to obtain a quantifier alternating \ASPQ program from the input $\qprg$, Algorithm~\ref{alg:cancel_weak} generates a sequence of programs by applying at each step one of the $col_{T}$ transformations.  With a little abuse of notation, we write that a program is of type $T$ ($T \in [1,5]$) if it satisfies the conditions for applying the rewriting $col_{T}$ defined above (cfr., Lemmas~\ref{lemma:uni-plain}-\ref{lemma:rewrite_fw}). For example, when type $T=1$ we check that the first two subprograms of $\qprg$ are plain and uniform so that $col_{1}$ can be applied to program $\qprg$.
In detail, at each iteration $s$, the innermost suffix program that is of current type $T$ is identified, say $\qprg_s^{\geq i}$. Then the next program $\Pi_{s+1}$ is built by replacing $\qprg_s^{>i}$ by $col_{T}(\qprg_s^{>i})$. Algorithm~\ref{alg:cancel_weak} terminates when no transformation can be applied, and returns the program $removeGlobal(\qprg_s)$.

\thmConvergence
{
Intuitively, the proof follows by observing that Algorithm~\ref{alg:cancel_weak} repeatedly simplifies the input by applying $col_T(\cdot)$ procedures ($T\in[1,5])$ until none can be applied. This condition happens when the resulting $\qprg'$ is plain alternating. Equi-coherence follows from Lemmas~\ref{lemma:uni-plain}-\ref{lemma:rewrite_fw}. Unless the innermost subprogram of $\qprg$ is not plain, no additional quantifier is added by  Algorithm~\ref{alg:cancel_weak}, so $nQuant(\qprg') \leq nQuant(\qprg) + 1$, hence the proof follows.

\begin{proof}
At each step $s$, Algorithm~\ref{alg:cancel_weak} searches for the innermost suffix subprogram  $\qprg_s^{\geq i}$ such that either $(i)$ $\qprg_s^{\geq i}$ begins with two consecutive quantifiers of the same type (i.e., it is of type 1,2 or 3), or $(ii)$ $\qprg_s^{\geq i}$ begins with a not plain subprogram followed by a quantifier alternating sequence of plain subprograms (i.e., it is of type 4 or 5). 
In case $(i)$, one of the subprocedures $col_1, col_2$, or $col_3$ is applied, which results in the computation of program $\Pi_{s+1}$ having one less pair of uniform subprograms (i.e., $nQuant(\Pi_{s+1}) = nQuant(\Pi_{s})-1$).
In case $(ii)$, one of the subprocedures $col_4, col_5$ is applied, which results in the computation of program $\Pi_{s+1}$ such that its $i$-th subprogram is plain. After applying 
$col_4, col_5$ we have that $nQuant(\Pi_{s+1}) \leq nQuant(\Pi_{s})+1$, indeed if $i=nQuant(\Pi_s)$ one more quantifier subprogram is added. So the algorithm continues until neither condition $(i)$ nor $(ii)$ holds. This happens when $\Pi_{s}$ is a plain quantifier alternating program. 
Note that, unless the innermost subprogram of $\qprg$ is not plain, no additional quantifier is added during the execution of Algorithm~\ref{alg:cancel_weak} (if anything, some may be removed), so $nQuant(\qprg') \leq nQuant(\qprg) + 1$.
\end{proof}

Additionally, it is easy to see that quantified answer set of existential programs can be preserved if only atoms of the first subprogram are made visible.

\begin{corollary}[ Quantified answer set preservation]
    Let $\qprg$ be an existential \ASPQW program of the form~(\ref{eq:aspqwprogram}) and $\qprg'$ be the result of the application of Algorithm~\ref{alg:cancel_weak} on $\qprg$. Then, $M \in QAS(\qprg)$ if and only if there exists $M' \in QAS(\qprg')$ such that $M'\cap HB_{\prg_1} = M$.
\end{corollary}

The Corollary above follows (straightforwardly) from Theorem~\ref{thm:convergence} because the only cases in which the first subprogram $\prg_1$ of $\qprg$ undergoes a modification during the rewriting is through a collapse operation, which, by definition, does not ``filter'' out any answer sets of the modified program. Since coherence is preserved by Theorem~\ref{thm:convergence}, a quantified answer set of $\qprg$ can be obtained from a quantified answer set of $\qprg'$ by projecting out atoms that are not in $HB_{\prg_1}$ (i.e., those not in the ``original'' $\prg_1$).

}


\section{Complexity issues}\label{sec:complexity}
{ In this section, we investigate the complexity of problems related to \ASPQW programs. We first study the complexity of the coherence problem. For that problem, global constraints can be ignored. Interestingly, the presence of local constraints leads to some unexpected phenomena. Next, we study the complexity of problems concerning membership of atoms in optimal answer sets. For this study, we restrict attention to existential programs with only global constraints.}

\thmUpperBound
\myProof{Let $\qprg'$ be the result of applying Algorithm~\ref{alg:cancel_weak} to $\qprg$. { Then,}
$\qprg'$ is a quantifier-alternating plain program with at most $n=nQuant(\qprg)+1$ quantifiers that is coherent iff $\qprg$ is coherent (Theorem~\ref{thm:convergence}). 
Thesis follows { from Theorem 3 in the paper by}~\cite{DBLP:journals/tplp/AmendolaRT19}.}

\thmLowerBound
The { result} above follows trivially from the observation that any quantifier-alternating \ASPQ program with $n$ quantifiers is a plain \ASPQW program where $\wprg=\emptyset$. 

{The lower and upper bounds offered by the two previous results do not meet in the general case. However, for some classes of \ASPQW programs they do, which leads to completeness results.
For instance, note that Algorithm~\ref{alg:cancel_weak} produces a quantifier-alternating plain \ASPQ program with at most $n$ quantifiers when the last subprogram is plain.} 

\corLastPlainSigmaN

{
\myProofSketch{The { assertion} follows from Theorem 3 { in the paper by}~\cite{DBLP:journals/tplp/AmendolaRT19} { and from properties of Algorithm~\ref{alg:cancel_weak}}.}


}

Note that, in plain \ASPQ (as well as in related formalisms { such as those considered by ~\cite{DBLP:journals/tcs/Stockmeyer76} and \cite{DBLP:journals/tplp/FandinnoLRSS21}}), the complexity of coherence correlates directly with the number of quantifier alternations~\citep{DBLP:journals/tplp/AmendolaRT19}.
%
Perhaps somewhat unexpectedly at first glance, it is not the case of \ASPQW. { There, when local constraints are present, one can ``go up'' one level with two consecutive quantifiers of the same kind.} This { phenomenon} is exemplified below.

\thmEEwSigmaTwo
\myProofSketch{{ \ (Membership) By applying Algorithm~\ref{alg:cancel_weak} on a uniform existential \ASPQW programs with two quantifiers where the program $P_2$ is not plain, we obtain an equi-coherent \ASPQ of the form $\exists P_1' \forall P_2':C'$. Thus, the membership to $\Sigma_2^P$ follows from Theorem 3 of~\cite{DBLP:journals/tplp/AmendolaRT19}.} %

\noindent Hardness is proved by a reduction of an existential 2QBF in DNF by adapting the QBF encoding in \ASPQ from Theorem 2 of \cite{DBLP:journals/tplp/AmendolaRT19}. In detail, a weak constraint in $P_2$ simulates the forall quantifier by preferring counterexamples that are later excluded by the final constraint $C$.}

The proof offers insights into this phenomenon, revealing that the second quantifier, governing optimal answer sets, essentially ``hides'' a universal quantifier.
The following result closes the picture for uniform plain programs with two existential quantifiers.

\propEEOrFFisNPorCONP
The result follows trivially from Lemma~\ref{lemma:uni-plain}, once we observe that one application of $col_1$ builds an equi-coherent program with one quantifier. 
%
%

\medskip
{
{ We will now turn our attention to problems involving optimal quantified answer sets.}

Observe that, as for the case of plain ASP, verifying the existence of an optimal quantified answer set has the same complexity as verifying the existence of a quantified answer set. Indeed, if a quantified answer set exists, there is certainly an optimal one. 
Thus, a more interesting task is to verify \textit{whether an atom $a$ belongs to some optimal quantified answer sets}.  (This is important as it supports brave reasoning as well as allows one to compute an optimal quantified answer set, if one exists). 

{ We will now study this problem for \textit{plain} \ASPQW programs with global constraints that seem to be especially relevant to practical applications.}
{ Similarly to what was proved} by \cite{DBLP:journals/tkde/BuccafurriLR00} for ASP, the task in question results in a jump in complexity.
Specifically, it elevates the complexity to being complete for $\Delta_n^P$ in the general case.

\thmWeightAndLevelsIsDelta

\myProofSketch{
    (Hardness) Hardness can be proved by resorting the observations used in the proof by \cite{DBLP:journals/tkde/BuccafurriLR00}. 
    More precisely, let $X_1,\ldots,X_n$ be disjoint sets of propositional variables, and $\Phi$ be a QBF formula of the form $\forall X_2 \exists X_3 \ldots \mathcal{Q} X_n\ \phi$, where each $Q\in\{\exists,\forall\}$, and $\phi$ is a formula over variables in $X_1,\ldots,X_n$ in 3-DNF if $n$ is even, otherwise it is in 3-CNF, and $X_1 = \{x_1,\ldots,x_m\}$. 
    Deciding whether the lexicographically minimum truth assignment $\tau$ of variables in $X_1$, such that $\forall X_2 \exists X_3 \ldots \mathcal{Q} X_n\ \phi_{\tau}$ is satisfied (assuming such $\tau$ exists), satisfies the condition $\tau(x_m)=\top$ is a $\Delta_{n+1}^P$-complete problem~\citep{DBLP:journals/tcs/Krentel92}.
    Such a problem can be encoded as a plain alternating \ASPQW program { $\Pi$ with $n$ quantifiers} such that an { atom $x_m$} appears in some optimal quantified answer set of $\qprg$ if and only if the answer to the problem is ``yes".
    
    (Membership) 
    As observed by~\cite{DBLP:journals/tkde/BuccafurriLR00} and~\cite{DBLP:journals/ai/SimonsNS02}, an optimal solution can be obtained with 
    binary search on the value of maximum possible cost, namely $k$.
    Since $k$ can be exponential in the general case, then an optimal quantified answer set of $\Pi$ can be obtained with a polynomial number of calls to the oracle in $\Sigma_n^P$, with $n=nQuant(\Pi)$. 
    Finally, an extra oracle call checks that the atom $a$ appears in some optimal quantified answer sets. 
}

Another interesting result regards a specific class of plain \ASPQW programs, namely those programs in which there is only one level and all the weights are the same.
In this particular case the complexity lowers to $\Theta_{n+1}^P$. 
Recall that, $\Theta_{n+1}^P$ is the class of problems that can be solved by a logarithmic number of calls to an oracle in $\Sigma_n^P$, that in the literature is also denoted by $\Delta_{n+1}^p[O(log\ m)]$~\citep{DBLP:conf/icalp/Wagner86}.
The next result shows \ASPQW can optimally encode optimization problems in this complexity class~\citep{DBLP:conf/icalp/Wagner86,DBLP:journals/siamcomp/Wagner90}, such as the Propositional Abduction Problem discussed in Section~\ref{sec:modeling}.

\thmNoLevelCostantWeightIsTheta

\begin{proof}
(Hardness) Let a QBF formula $\Phi$ be an expression of the form $\mathcal{Q}_1 X_1 \ldots \mathcal{Q}_n X_n \phi$, where $X_1,\ldots,X_n$ are disjoint sets of propositional variables, $\mathcal{Q}_i \in \{\exists,\forall\}$ for all $1\leq i\leq n$, $\mathcal{Q}_i \neq \mathcal{Q}_{i+1}$ for all $1\leq i < n$, and $\phi$ is a 3-DNF formula over variables in $X_1,X_2,\ldots,X_n$ of the form $D_1 \vee \ldots \vee D_n$, where each conjunct $D_i = l_1^i\wedge l_2^i\wedge l_3^i$, with $1\leq i\leq n$. 
A $k$-existential QBF formula $\Phi$ is a QBF formula where $n = k$ and $\mathcal{Q}_1 = \exists$.

Given a sequence of { $m$} $k$-existential QBF formulas { $\Phi_1,\ldots,\Phi_m$,} with $k$ being even and greater than or equal to $2$, and such that if $\Phi_j$ is unsatisfiable then also $\Phi_{j+1}$ is unsatisfiable, where $1\leq j < m$, deciding whether $v(\Phi_1,\ldots,\Phi_m) = max\{ j \mid 1\leq j \leq m \wedge \Phi_j\ \textit{is satisfiable}\}$ is odd is $\Theta_{k+1}$-complete~\citep{DBLP:journals/tkde/BuccafurriLR00}.

The above problem can be encoded into an \ASPQW program $\qprg$ such that a literal, namely $odd$, appears in some optimal quantified answer set of $\qprg$ if and only if $v(\Phi_1,\ldots,\Phi_m)$ is odd.
For simplicity, we introduce notation { for some sets of rules} that will be used in the construction of $\qprg$.
More precisely, given a QBF formula $\Phi$, $sat(\Phi)$ denotes the set of rules of the form $sat_{\Phi} \leftarrow l_1^i, l_2^i, l_3^i$, where $D_i = l_1^i\wedge l_2^i\wedge l_3^i$ is a conjunct in $\phi$; whereas for a set of variables $X_i = \{x_1^i,\ldots,x_n^i\}$ in $\Phi$, and an atom $a$, $choice(X_i,a)$ denotes the choice rule $\{x_1^i;\ldots;x_n^i\}\leftarrow a$.
We are now ready to construct the program $\qprg$.

First of all, we observe that all the formulas $\Phi_1,\ldots,\Phi_m$ have the same alternation of quantifiers. Thus, there is a one-to-one correspondence between the quantifiers in the QBF formulas and those in $\qprg$.
Let $\qprg$ be of the form $\Box_1\prg_1\Box_2\prg_2\ldots\Box_k \prg_k: C: \wprg$ where $\Box_i = \exists$ if $\mathcal{Q}_i = \exists$ in a formula $\Phi_j$, otherwise $\Box_i=\forall${ . The program} $\prg_1$ is of the form

{
\footnotesize
\[
    \prg_1 = \left\{
    \begin{array}{rllr}
        \{solve(1);\ldots;solve(m)\}=1 & \leftarrow& & \\
        unsolved(i)& \leftarrow & solve(j) & \forall\ j,i \in [1,\ldots,m] s.t.\ i>j\\
        odd & \leftarrow & solve(j) & \forall j \in [1,\ldots,m] s.t. \textit{ j is odd}\\
        choice(X_1^j,solve(j)) & & & \forall 1\leq j \leq m\\
    \end{array}
    \right\},
\]
\medskip
}

\noindent
{ while,}
for each $2\leq i \leq k$, the program $\prg_i$ is of the form
{
\footnotesize
\[
    \prg_i = \left\{
    \begin{array}{lr}
        choice(X_i^j,solve(j)) & \forall 1\leq j \leq m\\
    \end{array}
    \right\},
\]
}

\noindent
where each $X_i^j$ denotes the set of variables appearing in the scope of the $i$-th quantifier of the $j$-th QBF formula $\Phi_j$. { Finally,} the programs $\cprg$ and $\wprg$ are of the form

\begin{minipage}{.49\textwidth}
\footnotesize
\smallskip
\[
    \cprg = \left\{
    \begin{array}{rllr}
        &           & sat(\Phi_j) & \forall 1\leq j \leq m\\
        & \leftarrow& solve(j),\ \naf sat_{\Phi_j}  & \forall 1\leq j \leq m\\
    \end{array}
    \right\}
\]    
\smallskip
\end{minipage}
\hfill
\begin{minipage}{.49\textwidth}
\footnotesize
\[
    \wprg = \left\{
    \begin{array}{lr}
        \weakarr unsolved(i)\ [1@1,i] &\forall 1\leq i \leq m \\
    \end{array}
    \right\}.
\]    
\end{minipage}

Intuitively, the first choice rule in $\prg_1$ is used to guess one QBF formula, say $\Phi_j$, among the $m$ input ones, for which we want to verify the satisfiability. The guessed formula is encoded with the unary predicate $solve${, whereas, all the following formulas $\Phi_i$, with $i>j$, are marked as unsolved by means of the unary predicate $unsolved$.}

Then, $\prg_1$ contains different rules of the form $odd \leftarrow solve(j)$ for each odd index $j$ in $[1,m]$. 
Thus the literal $odd$ is derived whenever a QBF formula $\Phi_j$ in the sequence $\Phi_1,\ldots,\Phi_m$ is selected (i.e. $solve(j)$ is true) and $j$ is odd.
The remaining part of $\prg_1$ shares the same working principle of the following subprograms $\prg_i$, with $i\geq 2$. 
More precisely, for each QBF formula $\Phi_j$ in the sequence $\Phi_1,\ldots,\Phi_m$, they contain a choice rule over the set of variables quantified by the $i$-th quantifier of $\Phi_j$.
Note that the atom $solve(j)$ in the body of these choice rules guarantees that only one gets activated, and so the activated choice rule guesses a truth assignment for the variables in the $i$-th quantifier of $\Phi_j$.
Similarly, the constraint program $C$ contains, for each QBF formula $\Phi_j$ in the sequence $\Phi_1,\ldots,\Phi_m$, $(i)$ a set of rules that derives an atom $sat_{\phi_j}$ whenever the truth assignment guessed by the previous subprograms satisfies $\phi_j$, and $(ii)$ a strong constraint imposing that is not possible that we selected the formula $\Phi_j$ (i.e. $solve(j)$ is true) and $\phi_j$ is violated (i.e. $sat_{\Phi_j}$ is false).  
Thus, there exists a quantified answer set of $\qprg$ if and only if there exists a formula $\Phi_j$ in the sequence $\Phi_1,\ldots,\Phi_m$ such that $\Phi_j$ is satisfiable.
Since the program $\wprg$ contains the set of weak constraints of the form {$\weakarr unsolved(j)\ [1@1,j]$} for each $j \in [1,\ldots,m]$, the cost of each quantified answer set is given by the index $j$ of the selected formula. 
Thus, by minimizing the number of unsolved formulas
we are maximizing the index of the satisfiable formula $\Phi_j$. Thus, an optimal quantified answer set corresponds to a witness of coherence for a formula $\Phi_j$, s.t. for each $\Phi_{j'}$, with $j'>j$, $\Phi_{j'}$ is unsatisfiable.
By construction $odd$ is derived whenever $j$ is odd and so the { hardness} follows. 

(Membership) 
According to Theorem 3 of \cite{DBLP:journals/tplp/AmendolaRT19}, we know that the coherence of an existential plain alternating program with $n$ quantifiers falls within the complexity class $\Sigma_{n}^P$-complete.
By following { an} observation employed in the proofs by \cite{DBLP:journals/tkde/BuccafurriLR00}, { the cost of} an optimal solution can be obtained by 
binary search 
{ that terminates in} a logarithmic, { in the value of the maximum cost,} number of calls to an oracle in $\Sigma_{n}^P$ {that checks whether a quantified answer set with a lower cost with respect to the current estimate of the optimum exists.} { Once the cost of an optimal solution is determined, one more call to the oracle (for an appropriately modified instance), allows one to decide} the existence of an optimal solution containing $a$. Since each weak constraint has the same weight and the same level, then we can consider as the maximum cost the number of weak constraint violations. Thus, the number of oracle calls is at most logarithmic in the size of the problem and 
the membership follows.
\end{proof}

}

\section{Related Work}\label{sec:related}
Disjunctive ASP programs can be used to model problems in the second level of the PH using programming techniques, such as \textit{saturation}~\citep{DBLP:journals/amai/EiterG95,DBLP:journals/csur/DantsinEGV01}, but it is recognized that they are not intuitive. 
As a consequence, many language extensions have been proposed that expand the expressivity of ASP~\citep{DBLP:journals/tplp/BogaertsJT16,DBLP:journals/tplp/FandinnoLRSS21,DBLP:journals/tplp/AmendolaRT19}.
This paper builds on one of these, namely: Answer Set Programming with Quantifiers (\ASPQ) ~\citep{DBLP:journals/tplp/AmendolaRT19}.
\ASPQ extends ASP, allowing for declarative and modular modeling of problems of the entire PH~\citep{DBLP:journals/tplp/AmendolaRT19}. We expand \ASPQ with weak constraints to be able to model combinatorial optimization problems. In Section~\ref{sec:modeling}, we show the efficacy of  \ASPQW in modeling problems that would require cumbersome \ASPQ representations.

The two formalisms most closely related to \ASPQ are the stable-unstable semantics~\citep{DBLP:journals/tplp/BogaertsJT16}, and quantified answer set semantics~\citep{DBLP:journals/tplp/FandinnoLRSS21}. We are not aware of any extension of these that support explicitly weak constraints or alternative optimization constructs.
\cite{DBLP:journals/tplp/AmendolaRT19} and \cite{DBLP:journals/tplp/FandinnoLRSS21} provided an exhaustive comparison among ASP extensions for problems in the PH.

It is worth observing that \ASPQW extends \ASPQ by incorporating weak constraints, a concept originally introduced in ASP for similar purposes~\citep{DBLP:journals/tkde/BuccafurriLR00}.
Clearly, \ASPQW is a strict expansion of ASP, indeed it is easy to see that any ASP program $P$ is equivalent to a program of the form~(\ref{eq:aspqwprogram}) with only one existential quantifier, where $P_1=\rules{P}$, $\wprg = \weak{P}$, $\cprg = \emptyset$. 
Related to our work is also a formalism that has been proposed for handling preferences in ASP, called \textit{asprin}~\citep{DBLP:journals/ai/BrewkaDRS23}. \textit{asprin} is very handy in defining preferences over expected solutions, nonetheless, the complexity of main reasoning tasks in \textit{asprin} is at most $\Sigma_3^P$~\citep{DBLP:journals/ai/BrewkaDRS23}, with optimization tasks belonging at most to $\Delta_3^P$~\citep{DBLP:journals/ai/BrewkaDRS23}; thus, in theory, \ASPQW can be used to model more complex optimization problems (unless P=NP). \vspace{-0.3cm}


\section{Conclusion}\label{sec:conclusion}
{
We proposed an extension of \ASPQ enabling the usage of weak constraints for expressing complex problems in $\Delta_n^p$, called \ASPQW. We demonstrated \ASPQW's modeling capabilities providing suitable encodings for well-known complex optimization problems.
Also, we studied complexity aspects of \ASPQW, establishing upper and lower bounds for the general case, and revealing intriguing completeness results. 
Future work involves tightening the bounds from Theorem \ref{thm:general_complexity} for arbitrary $n$, extending \ASPQW to support subset minimality, 
and design a complexity-aware implementation for \ASPQW based on the translation of Section~\ref{sec:rewrite_to_aspq} and extending the system PyQASP~\citep{DBLP:journals/tplp/FaberMR23}.
}
\bibliographystyle{tlplike}
\bibliography{biblio}


\appendix
\section{Preliminaries on complexity classes}
In this section, we recall some basic definitions of complexity classes that are used to study the complexity of the introduced formalism. For further details about $NP$-completeness and complexity theory we refer the reader to dedicated literature~\citep{DBLP:books/daglib/0072413}.  
We recall that the classes $\Delta_k^P$, $\Sigma_k^P$, and $\Pi_k^P$ of the polynomial time hierarchy (PH)\citep{DBLP:journals/tcs/Stockmeyer76} are defined as follows (rf. \cite{DBLP:books/fm/GareyJ79}):
$$  \Delta_0^P = \Sigma_0^P = \Pi_0^P = P$$
and, for all $k>0$
$$  \Delta_{k+1}^P = P^{\Sigma_k^P},\ \ \Sigma_{k+1}^P = NP^{\Sigma_k^P},\ \ \Pi_{k+1}^P = coNP^{\Sigma_k^P},$$
where, $NP = \Sigma_1^P$, $coNP = \Pi_1^P$, and $\Delta_2=P^{NP}$. 

In general, $P^C$ (resp. $NP^C$) denotes the class of problems that can be solved in polynomial time on a deterministic (resp. nondeterministic) Turing machine with an oracle in the class $C$. 
Note that, the usage of an oracle $O\in C$ for solving a problem $\pi$ is referred to as a subroutine call, during the evaluation of $\pi$, to $O$. The latter is evaluated in a unit of time. 
Among such complexity classes, the classes $\Delta_k^P$, with $k \geq 2$, have been refined by the class $\Delta_k^P[O(\log n)]$ (also called $\Theta_k^P$), where the number of oracle calls is bounded by $O(\log n)$, with $n$ being the size of the input~\citep{DBLP:journals/tcs/Krentel92,DBLP:journals/siamcomp/Wagner90}.

\section{Modeling $\Delta_2^P$ in \ASPQW}
According to the complexity study carried out by \citeN{DBLP:journals/jacm/EiterG95}, given a PAP $\mathcal{A} = \langle V, T, H, M \rangle$ and a set $S \subseteq H$, the task of verifying if $S \in sol(\mathcal{A})$ is in $\Delta_2^P$. In particular, this task can be modeled with an \ASPQW $\qprg$ of the form $\exists \prg : \cprg$, where $\prg$ is not plain. 
More in detail, the program $\cprg$ contains only one constraint, which is $\leftarrow notEntail$ whereas the program $\prg$ is defined as follows: 

{
\footnotesize
\[
\prg = \left
        \{\begin{array}{rlll}
            v(x)                    &\leftarrow &                       &\forall\ x \in V \\
            lit(C_i,a,t)            &\leftarrow &                       &\forall\ a \in V \mid a \in C_i\\
            lit(C_i,a,f)            &\leftarrow &                       &\forall\ a \in V \mid \naf a \in C_i \\
            h(x)                    &\leftarrow &                       &\forall\ x \in H \\
            m(x)                    &\leftarrow &                       &\forall\ x \in M \\
            s(x)                    &\leftarrow &                       &\forall\ x \in S\\
            cl(X)                   &\leftarrow & lit(X,\_,\_)          &\\
            \{tau(X,t);tau(X,f)\}=1 &\leftarrow & v(X)                  & \\
            satCl(C)                &\leftarrow & lit(C,A,V), tau(A,V)  & \\
            unsatTS                 &\leftarrow & cl(C),\naf satCl(C)   & \\
            unsatTS                 &\leftarrow & s(X), tau(X,f)        & \\
                                    &\leftarrow & unsatTS               &\\
        \{tau'(X,t);tau'(X,f)\}=1   &\leftarrow & v(X)                  &\\
            satCl'(C)               &\leftarrow &lit(C,A,V), tau'(A,V)  &\\
            unsatTS'                &\leftarrow &cl(C),\naf satCl'(C)   &\\
            unsatTS'                &\leftarrow &s(X),\naf tau'(X,f)    &\\
            satTS                   &\leftarrow &\naf unsatTS'          &\\
            notEntail               &\leftarrow & satTS, m(X),tau'(X,f) &\\
                                    &\weakarr   & \naf notEntail\ \ [1@1]&
        \end{array}\right\}
\]
}

Intuitively, the program $\prg$ is used to verify that $T \cup S$ is consistent and $T \cup S$ entails the manifestation $M$. Verifying that $T\cup S \vDash M$ requires checking that for every truth assignment $M$ is satisfied whenever $T\cup S$ is satisfied. 
Thus, the main intuition is that, in this case, we can simulate this entailment check by means of weak constraints.
More precisely, an answer set $M \in \AS{\prg}$ contains a pair of truth assignments $\tau$ and $\tau'$ such that $\tau$ guarantee the consistency of $T \cup S$, whereas $\tau'$ either violates the entailment $T \cup S \vDash M$ and so $notEntail \in M$ or satisfies the entailment and so $notEntail \notin M$.
According to the weak constraint in $\prg$, an answer set containing $notEntail$ is preferred to an answer set not containing $notEntail$.
Thus, if there exists an answer set $M \in \OptAS{\prg}$ such that $notEntail \in M$ then $T \cup S \nvDash M$, and so, $S$ is not a solution for $\mathcal{A}$. 
Moreover, since $M$ is optimal then does not exist $M'\in \OptAS{\prg}$ such that $notEntail \notin M'$, and so, $\qprg$ is incoherent. 
Conversely, if there exists an answer set $M \in \OptAS{\prg}$ such that $notEntail \notin M$ then every answer set $M'\in \OptAS{\prg}$ is such that $notEntail \notin M'$. 
This means that $T \cup S \vDash$ holds and so $S$ is a solution for $\mathcal{A}$.
Since $notEntail \notin M$ then the program $\cprg$ is coherent and so, $\qprg$ is coherent.
\section{Stratified definition assumption}
In this section, we demonstrate that we can assume without loss of generality that \ASPQW programs satisfy the \textit{stratified definition assumption}.

\begin{definition}
Let $\qprg = \Box_1 \prg_1 \ldots \Box_n P_n$ be an \ASPQW program. $\qprg$ satisfies the stratified definition assumption if for each $1\leq i\leq n$, $\heads{\prg_i}\cap at(\prg_j) = \emptyset$, with $1\leq j < i$. 
\end{definition}

We demonstrate in the following that any \ASPQW program $\qprg$ can be transformed into a program $\qprg'$ such that $\qprg'$ satisfies the stratified definition assumption and is coherent whenever $\qprg$ is coherent. 

To this end, we recall that for an ASP expression $\epsilon$ and an alphanumeric string $s$, $clone^s(\epsilon)$ is the expression obtained by substituting all occurrences of each predicate $p$ in $\epsilon$ with $p^s$ that is a fresh predicate $p^s$ of the same arity.

It is easy to see that there is a one-to-one correspondence between answer sets of a program $P$ and its clone program $clone^s(P)$.

\begin{prop}\label{prop:clone_as}
    Let $\prg$ be an ASP program, and $s$ be an alphanumeric string. 
    Then $M \in AS(\prg)$ if and only if $clone^s(M) \in AS(clone^s(\prg))$.
\end{prop}

We now introduce the \textit{remap} function, which will be used to modify the signature of a subprogram.

\begin{definition}\label{def:remap}
Let $\prg_1$ and $\prg_2$ be two ASP programs, then $remap(\prg_2,\prg_1)$ is the program $clone^c(\prg_2) \cup \{clone^c(a)\leftarrow a \mid a \in \mathcal{B}_{\prg_1}\} \cup \{\leftarrow clone^c(a), \naf a \mid a \in \mathcal{B}_{\prg_1}\}$
\end{definition}
\begin{obs}\label{obs:fresh_atoms}
    Let $\prg_1$ and $\prg_2$ be two ASP programs, and $P = remap(\prg_2,\prg_1)$, then $\heads{\prg}\cap at(\prg_1)=\emptyset$.
\end{obs}

The above transformation has some important properties that follow from the \textit{splitting set}~(\cite{DBLP:conf/iclp/LifschitzT94}) theorem.

Given a program $\prg$, a \textit{splitting set}~(\cite{DBLP:conf/iclp/LifschitzT94}) for $\prg$ is a set of atoms $U \subseteq \mathcal{B}_{\prg}$, such that for every $r \in \prg$ such that $H_r \cap U \neq \emptyset$, $at(B_r) \subseteq U$. Let $U$ be a splitting set for $\prg$, $bot_U(P)$ denotes the set of rules $r \in P$ such that $at(r) \subseteq U$. 
Given two sets of atoms, $U$ and $X$, $e_U(P,X)$ denotes the set of rules obtained from rules $r \in P$ such that $(B_r^+ \cap U) \subseteq X$ and $X \setminus (at(B_r^-)\cap U) = X$, by removing from $B_r$ all those literals whose atom is in $U$.
\begin{thm}[\cite{DBLP:conf/iclp/LifschitzT94}]\label{splitting_set}
    Let $\prg$ be an ASP program and $U$ be a splitting set for $\prg$, then $M \in \AS{\prg}$ if and only if $M = X \cup Y$, where $X\in \AS{bot_U(\prg)}$, and $Y \in \AS{e_U(\prg\setminus bot_U(\prg)}$
\end{thm}

Now, we observe some properties of any pair of ASP programs.

\begin{lemma}\label{lemma:fix_to_split}
    Let $\prg_1,\prg_2$ be two ASP programs such that $at(\prg_1)\cap \heads{\prg_2} = \emptyset$, and $M_1 \in \AS{\prg_1}$.
    Then, each model $M$ such that $M \in \AS{\prg_2 \cup \fix{\prg_1}{M_1}}$ is of the form $M = M_1 \cup M_2$, where $M_2 \in \AS{e_{at(\prg_1)}(\prg_2,M_1)}$.

\end{lemma}
\begin{proof}
    Let $M_1 \in \AS{\prg_1}$, since $\prg_1,\prg_2$ are such that $at(\prg_1)\cap \heads{\prg_2} = \emptyset$, then $U = at(\prg_1)$ is a splitting set for program $P = \prg_2 \cup \fix{\prg_1}{M_1}$. 
    According to Theorem~\ref{splitting_set}, $M \in \AS{\prg}$ if and only if $M = X \cup Y$, where $X\in \AS{bot_U(\prg)}$, and $Y \in \AS{e_U(\prg\setminus bot_U(\prg),X)}$. 
    Now, we have that program $bot_U(\prg) = \fix{\prg_1}{M_1}$ that by definition admits only on model $M_1$. 
    Also, observe that $\prg\setminus bot_U(\prg) = \prg_2$, thus     
    $M = M_1 \cup M_2$ where $M_2\in \AS{e_U(\prg_2,M_1)}$.
\end{proof}

\begin{lemma}\label{lemma:remap}
    Let $\prg_1$ and $\prg_2$ be two ASP programs, $M_1 \in \AS{\prg_1}$, and $\prg_2' = \prg_2 \cup\fix{\prg_1}{M_1}$.
    Then $M_2 \in \AS{\prg_2'}$ if and only if $M= M_1 \cup clone^c(M_2)$ is an answer set of $\prg'=remap(\prg_2,\prg_1) \cup \fix{\prg_1}{M_1}$.
\end{lemma}
\begin{proof}
    Let $\prg_1$ and $\prg_2$ be two ASP programs and $M_1 \in \AS{\prg_1}$. For simplicity, we denote by $P = remap(\prg_2,\prg_1)$, and $\prg' = \prg \cup \fix{\prg_1}{M_1}$, where $M_1 \in \AS{\prg_1}$.
    Applying Observation~\ref{obs:fresh_atoms} to $\prg$, we have that $\heads{\prg}\cap at(\prg_1) = \emptyset$, and so from Lemma~\ref{lemma:fix_to_split} we have that each $M \in \AS{\prg'}$ is such that $M = M_1 \cup M_2'$, where $M_2' \in \AS{e_U(\prg,M_1)}$, where $U=at(\prg_1)$. 
    By definition $e_U(\prg,M_1) = e_U(remap(\prg_2,\prg_1),M_1) = clone^c(\prg_2) \cup \{clone^c(a) \leftarrow \mid a \in M_1\} \cup \{ \leftarrow clone^c(a) \mid a \in \mathcal{B}_{\prg_1}\setminus M_1\} = clone^c(\prg_2 \cup \fix{\prg_1}{M_1})$. 
    Since $\prg_2' = \prg_2 \cup \fix{\prg_1}{M_1}$ then $e_U(\prg,M_1) = clone^c(\prg_2')$. 
    Thus, each answer set $M_2' \in \AS{clone^c(\prg_2')}$ is such that $M_2' = clone^c(M_2)$ with $M_2 \in \AS{\prg_2'}$.     
    Thus, $M = M_1 \cup clone^c(M_2)$ iff $M_2 \in \AS{\prg_2'}$.
    %
    %
    %
    %
\end{proof}
\begin{corollary}\label{prop:remap_opt}
    Let $\prg_1$ and $\prg_2$ be two ASP programs, possibly with weak constraints, and $M_1 \in \AS{\prg_1}$, then $M\in \OptAS{remap(\prg_2,\prg_1)\cup \fix{\prg_1}{M_1}}$ if and only if $M = M_1 \cup clone^c(M_2)$, where $M_2 \in \OptAS{\prg_2 \cup \fix{\prg_1}{M_1}}$.
\end{corollary}
\begin{proof}
    Let $\prg_1$ and $\prg_2$ be two ASP programs, possibly with weak constraints, $M_1 \in \AS{\prg_1}$, $\prg = remap(\prg_2,\prg_1) \cup \fix{\prg_1}{M_1}$, and $\prg_2' = \prg_2 \cup \fix{\prg_1}{M_1}$.
    
    From Lemma~\ref{lemma:remap}, $M_2 \in \AS{\prg_2'}$ if and only if $ M = M_1 \cup clone^c(M_2) \in \AS{\prg}$.
    
    If $M_2 \notin \OptAS{\prg_2'}$ this means that there exists $M_2' \in \AS{\prg_2'}$ such that $M_2'$ dominates $M_2$, and so $M$ is dominated by $ M' = M_1 \cup clone^c(M_2')$. Thus $M \notin \OptAS{\prg}$.

    Conversely, if $M_2 \in \OptAS{\prg_2'}$ then for every $M_2' \in \AS{\prg_2'}$, $M_2$ is not dominated by $M_2'$ and so for every $M' = M_1 \cup clone(M_2') \in \AS{\prg}$, $M$ is not dominated by $M'$ and so $M\in\OptAS{\prg}$. 
\end{proof}

We are now ready to introduce our main program transformation.

\begin{definition}  
Given an \ASPQW program $\qprg$ of the form $\Box_1 \prg_1 \ldots\Box_n P_n: C$, and $1 < i \leq n$ we define the function $remap(\qprg,i)$ that computes  the \ASPQW program: 
$remap(\qprg,i) = \Box_1 \prg_1' \ldots\Box_n P_n': C'$, and $1\leq j \leq n$ such that:

{
\footnotesize
\[
    \prg_j' = \left
    \{\begin{array}{ll}
    \prg_j & j<=i \\
    remap(\prg_j, \prg_i) & j=i+1 \\
    clone^c(\prg_j) & j>i+1
    \end{array}\right.
\]
}

\noindent and $C' = clone^c(C)$.
\end{definition}

\begin{prop}\label{prop:qasp_remap}
    Let $\qprg$ be an \ASPQW program of the form $\Box_1 \prg_1 \ldots\Box_n P_n: C$, and $1\leq i \leq n$, then $\qprg$ is coherent if and only if $remap(\qprg,i)$ is coherent.
\end{prop}
\begin{proof}
The thesis follows from the definition of quantified answer set,  Proposition~\ref{prop:remap_opt} and Lemma~\ref{lemma:remap}. 
\end{proof}
\medskip\medskip\medskip\medskip\medskip
\begin{prop}\label{prop:recursive_remap}
Given an \ASPQW program $\qprg$. We define the sequence of \ASPQW programs 
\[
    \qprg_j = \left
    \{\begin{array}{ll}
    remap(\qprg,2) & j=2 \\
    remap(\qprg_{j-1}, j) & 2<j \leq n \\
    \end{array}\right.
\]
Then $\qprg_n$ is coherent iff $\qprg$ is coherent. 
\end{prop}

\begin{corollary}
Given an \ASPQW program $\qprg$ such that $\Box_1 = \exists$, then  $QAS(\qprg_n) = QAS(\qprg)$.
\end{corollary}
\begin{proof} It follows from Proposition~\ref{prop:recursive_remap} once we observe that the first programs of $\qprg_n$ and $\qprg$ are the same.
\end{proof}

Thus, without loss of generality, we assume that an \ASPQW program is of the form $\Box_1 \prg_1 \ldots \Box_n P_n$, where for each $1\leq i\leq n$, $\heads{\prg_i}\cap at(\prg_j) = \emptyset$, with $1\leq j < i$.

\section{Rewriting into plain \ASPQ}

In this section, we are going to prove the correctness of the transformation outlined in the corresponding section of the main paper. 
%
In what follows we assume that $\qprg$ is an \ASPQW program of the form \begin{equation}
\Box_1 \prg_1\ \Box_2 \prg_2\ \cdots\ \Box_n \prg_n :  \cprg : \wprg,
\label{eq:aspqwprogram}
\end{equation}
\noindent where, for each $i=1,\ldots,n$, $\Box_i \in \{ \exists^{st}, \forall^{st}\}$, $\prg_i$ is an ASP program possibly with weak constraints, 
$\cprg$ is a (possibly empty) stratified program with constraints, and $\wprg$ is a (possibly empty) set of weak constraints such that $B_{\wprg} \subseteq B_{\prg_1}$. 

We recall some useful definitions introduced in the main paper. 
Given program $\qprg$ of the form (\ref{eq:aspqwprogram}) we say that two consecutive subprograms $P_i$ and $P_{i+1}$ are \textit{alternating} if $\Box_i \neq \Box_{i+1}$, and are \textit{uniform} otherwise. 
A program $\qprg$ is \textit{quantifier-alternating} if $\Box_i \neq \Box_{i+1}$ for $1 \leq i < n$. 
A subprogram $P_i$ is \textit{plain} if it contains no weak constraints $\weak{P_i} = \emptyset$, and $\qprg$ is \textit{plain} if both all $P_i$ are plain, and $\wprg = \emptyset$.


\subsection{Rewriting uniform plain subprograms.}
First of all, we show how two plain uniform subprograms can be absorbed in a single equi-coherent subprogram. This is done by the transformation $col_1(\cdot)$ as follows.

\lemmacolone
\begin{proof} 
    The proof follows from the stratified definition assumption. In particular,
    since $at(\prg_1)\cap\heads{\prg_2} = \emptyset$ then $U=at(\prg_1)$ is a splitting set for $\prg = \prg_1 \cup \prg_2$, where $bot_U(\prg)=\prg_1$ and $\prg\setminus bot_U(\prg) = \prg_2$.
    From Theorem~\ref{splitting_set}, $M \in \AS{\prg}$ if and only if $M = M_1 \cup M_2$, where $M_1 \in \AS{\prg_1}$ and $M_2 \in \AS{e_U(\prg_2,M_1)}$, and so, from Lemma~\ref{lemma:fix_to_split}, $M \in \AS{\prg_2 \cup \fix{\prg_1}{M_1}}$. Thus, accordingly, the program $\prg$ preserves all the answer sets of $\prg_1$ and $\prg_2$.
    Since no weak constraints appear in $\prg_1$ and $\prg_2$, then $\AS{\prg} = \OptAS{\prg}$ and so the coherence is preserved. 
\end{proof}

\subsection{Rewriting uniform notplain-plain subprograms.}
Next transformations apply to pairs of uniform subprograms $P_1,P_2$ such that $P_1$ is not plain and $P_2$ is plain.
To this end, we recall the definition of $or(\cdot,\cdot)$ transformation. 
\begin{definition}\label{def:or_prg}
    Let $\prg$ be an ASP program, and $l$ be a fresh atom not appearing in $P$, then $or(P,l) = \{H_r\leftarrow B_r,\naf l \mid r \in P\}$
\end{definition}

\obsOr
Intuitively, if the fact $l\leftarrow$ is added to $or(\prg,l)$ then the interpretation $I=\{l\}$ trivially satisfies all the rules and is minimal, thus it is an answer set. On the other hand, if we add the constraint $\leftarrow l$, requiring that $l$ is false in any answer set, then the resulting program behaves precisely as $\prg$ since literal $\naf l$ is trivially true in all the bodies of the rules of the program.

We are now ready to introduce the next rewriting function $col_2(\cdot)$.

\defcoltwo

\lemmaColTwo
\begin{proof} [Proof (sketch)]
First observe that if $\prg_1$ is not coherent then both $\qprg$ (by definition) and $\exists \prg_1 \cup or(\prg_2,unsat)$ (for the splitting theorem) are not coherent, and thus (by definition of quantified answer set) also $col_2(\qprg)$ is not coherent. 
Next observe that, if $\prg_1$ is coherent, an optimal answer set of $\prg_1 \cup or(\prg_2,unsat)$ contains $unsat$ only if any optimal answer set $m$ of $P_1$ is such that $\prg_2 \cup \fix{\prg_1}{m}$ is incoherent. In this case, $\qprg$ is not coherent, and since the subprogram following $P_2$ contains the constraint $\leftarrow unsat$, also $col_2(\qprg)$ is not coherent. 
On the other hand, if $unsat$ is false in any optimal answer set of the first subprogram of $col_2(\qprg)$, then (if $n\geq 3$) $P_3\prime$ behaves as $P_3$, and the constraint $\leftarrow unsat$ occurring in the next subprogram (i.e., $P_4$ or $C$) is trivially satisfied. Thus (by the definition of quantified answer set) $col_2(\qprg)$ is coherent whenever $\qprg$ is coherent.
\end{proof}
\begin{proof}
Let $\prg = \prg_1 \cup or(\prg_2,unsat) \cup W$, where $W = \{\{unsat\}\leftarrow\} \cup \{\weakarr unsat\ [1@l_{min}-1]\}$, with $l_{min}$ being the lowest level in $\weak{\prg_1}$ and $unsat$ is a fresh symbol not appearing in $\qprg$.

Since $\prg_1$ and $\prg_2$ satisfy the stratified definition assumption, then, from the splitting theorem, each answer set of $\prg$ can be computed by fixing any answer set $M_1 \in \AS{\prg_1}$ in the program $or(\prg_2,unsat) \cup W$.
In turn, $M \in \AS{\prg}$ if and only if $\exists M_1 \in \AS{\prg_1}$ and $M \in \AS{or(\prg_2,unsat) \cup W \cup \fix{\prg_1}{M_1}}$.

From the definition of quantified answer set, if $\prg_1$ is incoherent then $\qprg$ is incoherent.
Analogously, $\prg$ is incoherent and so, also $col_2(\qprg)$ is incoherent.

On the other hand, i.e. $\prg_1$ is coherent, then $\AS{\prg} = \{M_1 \cup \{unsat\} \mid M_1 \in \AS{\prg_1}\} \cup \{M \mid \exists M_1 \in \AS{\prg_1} \wedge M \in \AS{\prg_2 \cup \fix{\prg_1}{M_1}}\}$.

Let $M$ in $\AS{\prg}$ such that $M$ is obtained from $M_1 \in \AS{\prg_1}\setminus\OptAS{\prg_1}$ (i.e. either $M = M_1 \cup \{unsat\}$ or $M \in \AS{\prg_2\cup\fix{\prg_1}{M_1}}$). 
In this case, we know that there exists $M' = M_1' \cup \{unsat\} \in \AS{\prg}$, where $M_1' \in \OptAS{\prg_1}$, and so, since weak constraints in $\prg_1$ are defined at the highest priority level, then $M$ is dominated by $M'$. 
In turn $M \notin \OptAS{\prg}$.

Let $M = M_1 \cup \{unsat\} \in \AS{\prg}$, with $M_1 \in \OptAS{\prg}$. 
In this case, according to the weak constraint $\weakarr unsat\ [1@l_{min}-1]$, $M$ is dominated only by any $M' \in \AS{\prg}$ such that $M' \in \AS{\prg_2 \cup \fix{\prg_1}{M_1'}}$, with $M_1' \in \OptAS{\prg_1}$, since $M'$ does not violate the weak constraint $\weakarr unsat\ [1@l_{min}-1]$.

Thus, if there exists $M \in \OptAS{\prg}$ such that $unsat \in M$ then $\OptAS{\prg} = \{M_1 \cup \{unsat\} \mid M_1 \in \OptAS{\prg_1}\}$. This means that for every $M_1 \in \OptAS{\prg_1}$ the program $\prg_2 \cup \fix{\prg_1}{M_1}$ is incoherent and so, from the definition of quantified answer set, $\qprg$ is incoherent.
Since $unsat$ is true, in this case, in every optimal answer set of $\prg$ and the subprogram following $P_2$ contains the constraint $\leftarrow unsat$ then also $col_2(\qprg)$ is incoherent.

On the other hand, if there exists $M \in \OptAS{\prg}$ such that $unsat\notin M$ then $\OptAS{\prg} = \{M \mid M \in \AS{\prg_2\cup\fix{\prg_1}{M_1}} \wedge M_1 \in \OptAS{\prg_1}\}$.
In this case, since $unsat$ is false in any optimal answer set of $\prg$, then (if $n\geq 3$) $P_3\prime$ behaves as $P_3$, and the constraint $\leftarrow unsat$ occurring in the next subprogram (i.e., $P_4$ or $C$) is trivially satisfied. 
Thus (by the definition of quantified answer set) $col_2(\qprg)$ is coherent whenever $\qprg$ is coherent.
\end{proof}

\defColThree
    
    
\medskip\medskip\medskip
\lemmaColThree
\begin{proof} [Proof (sketch)]
The proof follows the same idea used for proving Lemma~\ref{lemma:collapse_ew_e}.
In this case, if $\prg_1$ is not coherent then $\qprg$ (by definition) is coherent. At the same time $\forall \prg_1 \cup or(\prg_2,unsat)$ (for the splitting theorem) is also incoherent, and thus (according to the coherence of \ASPQW) also $col_2(\qprg)$ is coherent. 
On the other hand, if $\prg_1$ is coherent then an optimal answer set of $\forall \prg_1 \cup or(\prg_2,unsat)$ contains $unsat$ only if any optimal answer set $m$ of $P_1$ is such that $\prg_2 \cup \fix{\prg_1}{m}$ is incoherent. In this case, $\qprg$ is coherent, and, according to the definition of optimal answer set, every optimal answer set of $\forall \prg_1 \cup or(\prg_2,unsat)$ contains $unsat$. Thus, the subprogram following $P_2$ are trivially satisfied ($or(C,unsat)$, $or(P_3,unsat)$ or $\forall \prg_4 \cup \{\leftarrow unsat\}$), and so also $col_2(\qprg)$ is coherent. 
On the other hand, if $unsat$ is false in any optimal answer set of the first subprogram of $col_2(\qprg)$, then (if $2\leq n\leq 3$) $or(C,unsat)$ behaves as $C$, (if $n\geq 3$) $P_3'$ behaves as $P_3$, and the constraint $\leftarrow unsat$ occurring in the next subprogram (i.e. $P_4$) is trivially satisfied. Thus (according to the coherence of \ASPQW) $col_2(\qprg)$ is coherent whenever $\qprg$ is coherent.
\end{proof}
\begin{proof}
Let $\prg = \prg_1 \cup or(\prg_2,unsat) \cup W$, where $W = \{\{unsat\}\leftarrow\} \cup \{\weakarr unsat\ [1@l_{min}-1]\}$, with $l_{min}$ being the lowest level in $\weak{\prg_1}$ and $unsat$ is a fresh symbol not appearing in $\qprg$.

As it has been observed in the proof of Lemma~\ref{lemma:collapse_ew_e}, since $\prg_1$ and $\prg_2$ satisfy the stratified definition assumption, then, from the splitting theorem, each answer set of $\prg$ can be computed by fixing any answer set $M_1 \in \AS{\prg_1}$ in the program $or(\prg_2,unsat) \cup W$.
In turn, $M \in \AS{\prg}$ if and only if $\exists M_1 \in \AS{\prg_1}$ and $M \in \AS{or(\prg_2,unsat) \cup W \cup \fix{\prg_1}{M_1}}$.

According to the coherence of \ASPQW programs, if $\prg_1$ is incoherent then $\qprg$ is coherent.
Analogously, $\prg$ is incoherent and so, also $col_2(\qprg)$ is coherent.

On the other hand, i.e. $\prg_1$ is coherent, then $\AS{\prg} = \{M_1 \cup \{unsat\} \mid M_1 \in \AS{\prg_1}\} \cup \{M \mid \exists M_1 \in \AS{\prg_1} \wedge M \in \AS{\prg_2 \cup \fix{\prg_1}{M_1}}\}$.

Thus, from the observation made in the proof of Lemma~\ref{lemma:collapse_ew_e}, if there exists $M \in \OptAS{\prg}$ such that $unsat \in M$ then $\OptAS{\prg} = \{M_1 \cup \{unsat\} \mid M_1 \in \OptAS{\prg_1}\}$. 
This means that for every $M_1 \in \OptAS{\prg_1}$ the program $\prg_2 \cup \fix{\prg_1}{M_1}$ is incoherent and so, from the coherence of \ASPQW, $\qprg$ is coherent.
Since $unsat$ is true in every optimal answer set of $\prg$ then, in this case, the subprogram following $P_2$ (i.e. $or(C,unsat)$, $or(P_3,unsat)$ or $\forall \prg_4 \cup \{\leftarrow unsat\}$) are trivially satisfied, and so also $col_2(\qprg)$ is coherent.

On the other hand, if there exists $M \in \OptAS{\prg}$ such that $unsat\notin M$ then $\OptAS{\prg} = \{M \mid M \in \AS{\prg_2\cup\fix{\prg_1}{M_1}} \wedge M_1 \in \OptAS{\prg_1}\}$.
In this case since $unsat$ is false in any optimal answer set of $\prg$, then (if $2\leq n\leq 3$) $or(C,unsat)$ behaves as $C$, (if $n\geq 3$) $P_3'$ behaves as $P_3$, and the constraint $\leftarrow unsat$ occurring in the next subprogram (i.e. $P_4$) is trivially satisfied. Thus (according to the coherence of \ASPQW) $col_2(\qprg)$ is coherent whenever $\qprg$ is coherent.
\end{proof}


\subsection{Rewrite subprograms with weak constraints.}
The following transformations have the role of eliminating weak constraints from a subprogram by encoding the optimality check in the subsequent subprograms.
To this end, we recall the definition of $check(\cdot)$ transformation that is useful for simulating the cost comparison of two answer sets of an ASP program $P$.

\defTranslateWeak
    
    

Thus, the first two rules compute in predicate $cl_{\prg}$ the cost of an answer set of $P$ w.r.t. its weak constraints, and the following two rules do the same for $clone^o(P)$.
Then, the last four rules derive $dom_{\prg}$ for each answer set of $P$ that is dominated by $clone^o(P)$. 

\begin{obs} 
    Let $\prg$ be an ASP program with weak constraints, and $M_1,M_2 \in \AS{\prg}$, then $M_1$ is dominated by $M_2$ if and only if $check(\prg)\cup\fix{\prg}{M_1}\cup clone^o(\fix{\prg}{M_2})$ admits an answer set $M$ such that $dom_{\prg} \in M$.
\end{obs}

\defColFour
\lemmaColFour
    
    
\begin{proof}[Proof (sketch)]
Intuitively, $col_4(\qprg)$ is structured in such a that if there exists an answer set $m_1$ of $P_1$ (any, also not optimal ones), such that for any other answer set of $m_i^o$ of $P_1$ (computed by cloning $P_1$ in the second subprogram of $col_4(\qprg)$) either $dom_{P_1}$ is derived or $m_1$ is optimal. In the first case, if $n\geq 2$ $dom_{P_1}$ inhibits the rules of $P_2$, and the next subprogram discards $m_1$, as expected since it is not optimal. In the second case, the next subprograms of $col_4(\qprg)$ behave as those of $\qprg$, and the constraint $\leftarrow dom_{P_1}$ occurring next is trivially satisfied. Thus, $m_1$ is a quantified answer set of $col_4(\qprg)$ only if $m_1$ is a quantified answer set of $\qprg$.
\end{proof}
\begin{proof}
Let $\qprg$ be an existential alternating \ASPQW program such that all subprograms are plain except the first.
First of all we observe that $\AS{\rules{\prg_1}} = \AS{\prg_1} \subseteq \OptAS{\prg_1}$. 
Thus, if $\prg_1$ is incoherent then $\AS{\rules{\prg_1}} = \AS{\prg_1} = \emptyset$, and so, also $\rules{\prg_1}$ is incoherent.
Indeed, from the definition of quantified answer set, both $\qprg$ and $col_4(\qprg)$ are incoherent.

On the other hand, let $M_1 \in \AS{\prg_1}$ and $\prg$ denotes the second subprogram of $col_4(\qprg)$ ($\prg = clone^o(\rules{\prg_1})\cup check(\prg_1)$, if $n=1$, otherwise $\prg = clone^o(\rules{\prg_1})\cup check(\prg_1) \cup or(\prg_2,dom_{\prg_1})$).

If $M_1 \notin \OptAS{\prg_1}$ then $M_1$, from the definition of quantified answer set, is not a quantified answer set for $\qprg$.
Since $M_1$ is not optimal then we know that there exists $M_1' \in \OptAS{\prg_1}$ such that $M_1$ is dominated by $M_1'$, and so there exists $M \in \AS{\prg\cup\fix{\rules{\prg_1}}{M_1}}$ such that $clone^o(M_1')\subseteq M$ and $dom_{\prg_1}\in M$.
Thus, $M$ violates the strong constraint $\leftarrow dom_{\prg_1}$ in the following subprograms and so, from the definition of quantified answer set, $M_1$ is not a quantified answer set of $col_4(\qprg)$.

Conversely, if $M_1 \in \OptAS{\prg_1}$ then for every $M \in \AS{\prg\cup\fix{\rules{\prg_1}}{M_1}}$, $dom_{\prg_1}\notin M$ and so the constraint $\leftarrow dom_{\prg_1}$ added in the subsequent subprograms is trivially satisfied. Since atoms in $clone^o(\rules{\prg_1})$ and $check(\prg_1)$ do not appear anywhere else then they do not affect the coherence of $col_4(\qprg)$ and so $M_1$ is a quantified answer set of $col_4(\qprg)$ whenever $M_1$ is a quantified answer set for $\qprg$.
\end{proof}

\defColFive
\lemmaColFive
    
    
The proof of Lemma~\ref{lemma:rewrite_fw} can be established using a dual argument with respect to that employed for Lemma~\ref{lemma:rewrite_ew}.

\subsection{Translate \ASPQW to \ASPQ.}
\begin{algorithm}[t!]\footnotesize
\caption{Rewrite from \ASPQW to \ASPQ}
    \label{alg:cancel_weak}
    \SetKwInOut{Input}{Input}
    \SetKwInOut{Output}{Output}
    \SetKwRepeat{Do}{do}{while}
    \Input{An \ASPQW program $\Pi$}
    \Output{A quantifier-alternating \ASPQ program}
    \Begin{
        $s$ := $0$; ~~ $\qprg_0$ := $\qprg$\\
        \Do{$stop \neq \top$}{
            $stop$ := $\top$\\
            \ForAll{$ProgramType \in [1,5]$}{
                Let $i \in [1,n]$ be the largest index such that $\qprg_s^{>i}$ is of the type $ProgramType$ \\ 
                \If{ $i \neq \bot$ }{ 
                    $\qprg_{s+1}$ := $replace(\qprg_s,i,col_{ProgramType}(\qprg_s^{\geq i}))$\\
                    $s$ := $s+1$; \\
                    $stop$ := $\bot$\\
                    $break$  \hfill // go to line 12
                }
            }
        }
        \Return{removeGlobal($\qprg_s$)}
    }
\end{algorithm}
Algorithm~\ref{alg:cancel_weak} defines a procedure for rewriting an \ASPQW program $\qprg$ into an \ASPQ program $\qprg'$, made of at most $n+1$ alternating quantifiers, such that $\qprg$ is coherent if and only if $\qprg'$ is coherent. 
We recall that in Algorithm~\ref{alg:cancel_weak}, we make use of some (sub)procedures and dedicated notation.
More precisely, for a program $\qprg$ of the form (\ref{eq:aspqwprogram}), $\qprg^{\geq i}$ denotes the \textit{i-th suffix program} $\Box_i \prg_i \ldots \Box_n \prg_n : C$, with $1\leq i \leq n$. (i.e., the one obtained from $\qprg$ removing the first $i-1$ quantifiers and subprograms). 
Moreover, procedure $removeGlobal(\qprg)$ builds an \ASPQ program from a plain one in input (roughly, it removes the global constraint program $\wprg$).
Given two programs $\qprg_1$ and $\qprg_2$, $replace(\qprg_1,i,\qprg_2)$ returns the \ASPQW program obtained from $\qprg_1$ by replacing program $\qprg_1^{\geq i}$ by $\qprg_2$, 
for example $replace(\exists P_1 \forall P2 \exists P_3 : C, 2, \exists P_4 :C)$ returns $\exists P_1 \exists P_4 :C$. With a little abuse of notation, we write that a program is of type $T \in [1,5]$ if it satisfies the conditions for applying the rewriting $col_{T}$ defined above (cfr., Lemmas~\ref{lemma:uni-plain}-\ref{lemma:rewrite_fw}). For example, when type $T=1$ we check that the first two subprograms of $\qprg$ are plain and uniform so that $col_{1}$ can be applied to program $\qprg$.

In order to obtain a quantifier alternating \ASPQ program from the input $\qprg$, Algorithm~\ref{alg:cancel_weak} generates a sequence of programs by applying at each step one of the $col_{T}$ transformations. In particular, at each iteration $s$, the innermost suffix program, say $\qprg_s^{\geq i}$, that is of current type $T$ is identified. Then the next program $\Pi_{s+1}$ is built by replacing $\qprg_s^{>i}$ by $col_{T}(\qprg_s^{>i})$. Algorithm terminates when no transformation can be applied, and returns the program $removeGlobal(\qprg_s)$.

\thmConvergence
\begin{proof}[Proof. (Sketch)]
Algorithm~\ref{alg:cancel_weak} repeatedly simplifies the input by applying $col_T(\cdot)$ procedures ($T\in[1,5])$ until none can be applied. So, the results follow from the Lemmas~\ref{lemma:uni-plain}-\ref{lemma:rewrite_fw} that ensures the input can be converted to an equi-coherent plain \ASPQW program. Note that, unless the innermost subprogram of $\qprg$ is not plain, no additional quantifier is added during the execution of Algorithm~\ref{alg:cancel_weak} (if anything, some may be removed), so $nQuant(\qprg') \leq nQuant(\qprg) + 1$.
\end{proof}

\begin{proof}
At each step $s$, Algorithm~\ref{alg:cancel_weak} searches for the innermost suffix subprogram  $\qprg_s^{\geq i}$ such that either $(i)$ $\qprg_s^{\geq i}$ begins with two consecutive quantifiers of the same type (i.e., it is of type 1,2 or 3), or $(ii)$ $\qprg_s^{\geq i}$ begins with a not plain subprogram followed by a quantifier alternating sequence of plain subprograms (i.e., it is of type 4 or 5). 
In case $(i)$, one of the subprocedures $col_1, col_2$, or $col_3$ is applied, which results in the computation of program $\Pi_{s+1}$ having one less pair of uniform subprograms (i.e., $nQuant(\Pi_{s+1}) = nQuant(\Pi_{s})-1$).
In case $(ii)$, one of the subprocedures $col_4, col_5$ is applied, which results in the computation of program $\Pi_{s+1}$ such that its $i$-th subprogram is plain. After applying 
$col_4, col_5$ we have that $nQuant(\Pi_{s+1}) \leq nQuant(\Pi_{s})+1$, indeed if $i=nQuant(\Pi_s)$ one more quantifier subprogram is added. So the algorithm continues until neither condition $(i)$ nor $(ii)$ holds. This happens when $\Pi_{s}$ is a plain quantifier alternating program. 
Note that, unless the innermost subprogram of $\qprg$ is not plain, no additional quantifier is added during the execution of Algorithm~\ref{alg:cancel_weak} (if anything, some may be removed), so $nQuant(\qprg') \leq nQuant(\qprg) + 1$.
\end{proof}

\section{Complexity issues}\label{sec:complexity}
In this section, we recall the complexity results related to verifying the coherence of \ASPQW programs and provide full proof for completeness results.  

\thmUpperBound
\begin{proof}
Let $\qprg'$ be the result of applying Algorithm~\ref{alg:cancel_weak} to $\qprg$. Then,
$\qprg'$ is a quantifier-alternating plain program with at most $n=nQuant(\qprg)+1$ quantifiers that is coherent iff $\qprg$ is coherent (Theorem~\ref{thm:convergence}). 
Thesis follows { from Theorem 3 in the paper by}~\cite{DBLP:journals/tplp/AmendolaRT19}.
\end{proof}

\thmLowerBound
\begin{proof}
The proof trivially follows from the observation that any quantifier-alternating \ASPQ program with $n$ quantifiers is a plain \ASPQW program where $\wprg=\emptyset$. 
\end{proof}

\corLastPlainSigmaN
\begin{proof}
The statement follows from Theorem 3 of~\cite{DBLP:journals/tplp/AmendolaRT19}.

\noindent(Membership) Let $\qprg$ be an \ASPQW program where the last subprogram is plain.
We observe that by applying Algorithm~\ref{alg:cancel_weak} we obtain a quantifier-alternating plain \ASPQ program $\qprg'$ with at most $n+1$ quantifiers such that $\qprg'$ is coherent if and only if $\qprg$ is coherent.
Since the only case in which an extra quantifier is added is when the last subprogram is not plain (i.e. $\weak{\prg_n}\neq\emptyset$) then $nQuant(\qprg')\leq n$ then the membership follows.

\noindent(Hardness) The hardness trivially follows by observing that any quantifier-alternating \ASPQ program with $n$ quantifiers is trivially encoded as a plain \ASPQW program where $\wprg=\emptyset$.
\end{proof}

Note that, in plain \ASPQ (as well as in related formalisms~\citep{DBLP:journals/tcs/Stockmeyer76,DBLP:journals/tplp/FandinnoLRSS21}), the complexity of coherence correlates directly with the number of quantifier alternations~\citep{DBLP:journals/tplp/AmendolaRT19}.
%
Perhaps somewhat unexpectedly at first glance, it is not the case of \ASPQW, where one can ``go up one level'' with two consecutive quantifiers of the same kind. This observation is exemplified below.

\thmEEwSigmaTwo
\begin{proof}[Proof. (Sketch)]
Membership follows from Theorem~\ref{thm:general_complexity}. %
\noindent Hardness is proved by a reduction of an existential 2QBF in DNF by modifying the QBF encoding in \ASPQ presented in Theorem 2 of \cite{DBLP:journals/tplp/AmendolaRT19}. In particular, a weak constraint in $P_2$ simulates the forall quantifier by preferring counterexamples that are later excluded by the final constraint $C$.    
\end{proof}
\begin{proof}
    \noindent(Hardness) Let us consider a QBF formula $\Phi = \exists X_1 \forall X_2 \phi$, where $X_1,X_2$ are two disjoint sets of propositional variables, and $\phi$ is a 3-DNF formula over variables in $X_1,X_2$ of the form $D_1 \vee \ldots \vee D_n$, where each conjunct $D_i = l_1^i\wedge l_2^i\wedge l_3^i$, with $1\leq i\leq n$. It is known that the task of verifying the satisfiability of $\Phi$ is a $\Sigma_2^P$-complete problem~\citep{DBLP:journals/tcs/Stockmeyer76}, thus, we reduce $\Phi$ to an \ASPQW program $\qprg$ of the form $\exists P_1 \exists P_2: C$ where

{
\vspace{.3cm}
\footnotesize
\begin{minipage}{.25\textwidth}
\centering
    \[
    \prg_1 = \left\{
    \begin{array}{lr}
        \{x\}\leftarrow & \forall x \in X_1 \\
    \end{array}
    \right\}
    \]
\end{minipage}
\hfill
\begin{minipage}{.35\textwidth}
\centering
    \[
    \prg_2 = \left\{
    \begin{array}{lr}
         \{x\}\leftarrow&  \forall x \in X_2\\
         sat \leftarrow l_1^i,l_2^i,l_3^i& \forall D_i \in \phi\\
         \weakarr sat [1@1]
    \end{array}
    \right\}
    \]
\end{minipage}
\hfill
\begin{minipage}{.2\textwidth}
\centering
    \[
    C=\left\{
    \leftarrow \naf sat
    \right\}
    \]
\end{minipage}
\vspace{.3cm}
}

    We recall that 2-QBF formula $\Phi$ of the form $\exists X_1 \forall X_2\ \phi$ where is satisfiable if and only if there exists a truth assignment $\tau_1$ for variables in $X_1$ such that for every truth assignment $\tau_2$ of variables in $ X_2$, the formula $\phi$ is satisfied w.r.t. $\tau_1$ and $\tau_2$ (i.e. at least a conjunct in $\phi$ is true w.r.t. $\tau_1$ and $\tau_2$).
    
    To this end, the program $P_1$ encodes the truth assignments of variables in $X_1$ by means of a choice rule for each $x \in X_1$.
    Analogously, $P_2$ encodes the truth assignments of variables in $X_2$, by means of a choice rule for each $x \in X_2$, and checks whether $\phi$ is satisfied or not by means of a rule for each conjunct $D_i$, that derives the atom $sat$ whenever $D_i$ is true.

    Thus, if $\Phi$ is satisfiable then there exists $M_1 \in \AS{P_1} = \OptAS{P_1}$, such that $M_1$ encodes $\tau_1$, and for every answer set $M_2 \in \AS{P_2'}$, $sat\in M_2$, where $P_2' = P_2\cup\fix{P_1}{M_1}$.
    
    Since $sat$ appears in every $M_2 \in \AS{P_2'}$ then for every $M_2 \in \AS{P_2'}$, $\levelCost{P_2'}{M_2}{1}=1$, and so, $\AS{P_2'} = \OptAS{P_2'}$.
    Thus, for every $M_2 \in \OptAS{P_2'}$, $C \cup \fix{P_2'}{M_2}$ is coherent and so, $M_1 \in QAS(\qprg)$.

    Conversely, if $\Phi$ is unsatisfiable then for every truth assignment $\tau_1$ over variables $X_1$, there exists a truth assignment $\tau_2$ over variables $X_2$ such that $\phi$ is unsatisfiable w.r.t. $\tau_1$ and $\tau_2$. 
    This means that each conjunct of $\phi$ is false w.r.t. $\tau_1$ and $\tau_2$.
    Thus, for every $M_1 \in \AS{P_1}$ there exists $M_2 \in \AS{P_2'}$ such that $sat \notin M_2$, with $P_2' = P_2 \cup \fix{P_1}{M_1}$.
    This means that $\levelCost{P_2'}{M_2}{1}=0$, and so, $\OptAS{P_2'} = \{M_2 \in \AS{P_2'}\mid sat \notin M_2\}$. 
    Thus, there exists $M_2 \in \OptAS{P_2'}$, $C\cup\fix{P_2'}{M_2}$ is incoherent and so $\qprg$ is incoherent.

    At this point, it is easy to see that we can trivially add in the program $\prg_1$ a weak constraint of the form $\weakarr a\ [1@1]$ where $a$ is a fresh atom not appearing anywhere else without affecting the optimal answer set of $\prg_1$.  
    Thus, this hardness holds both if $\prg_1$ is plain or not.

    \noindent(Membership)
    Let $\qprg$ be an \ASPQW program of the form $\exists \prg_1 \exists\prg_2:\cprg$ where $\prg_2$ is not plain.
    By applying Algorithm~1 on $\qprg$, no matter if $\prg_1$ contains weak constraints or not, we obtain an existential \ASPQ program $\qprg'$, made of two alternating quantifiers, such that $\qprg$ is coherent if and only if $\qprg'$ is coherent. Since verifying the coherence of $\qprg'$ is in $\Sigma_2^P$~\citep{DBLP:journals/tplp/AmendolaRT19}-complete then verifying the coherence of $\qprg$ is also in $\Sigma_2^P$.
\end{proof}
The proof provides insights into this phenomenon. Indeed, the second quantifier, the one over optimal answer sets, ``hides'' a universal quantifier. 
%

\begin{thm}\label{thm:ew_c}
    Deciding coherence of uniform existential \ASPQW programs with at most two quantifiers (i.e. $n\leq2$) such that only $\prg_1$ contains weak constraints (i.e. $\weak{\prg_1}\neq\emptyset$ and  (if $n=2$) $\weak{\prg_2}=\emptyset$) is $\Delta_2^P$-complete.
\end{thm}
\begin{proof}
    From Lemma~\ref{lemma:collapse_ew_e} we observe that each program $\qprg = \exists \prg_1 \exists \prg_2:C$, where $\prg_1$ is not plain and $\prg_2$ is plain, can be transformed into an \ASPQW program $\qprg'=col_2(\qprg)$ such that $\qprg$ is coherent if and only if $\qprg'$ is coherent. Thus it is sufficient to prove that the statement holds for $n=1$.


    \noindent(Hardness) Given a program $\prg$, it is known that the task of verifying that an atom $a\in \HB$ appears in some optimal answer sets is $\Delta_2^P$-complete, for a normal program with weak constraints~\citep{DBLP:journals/tkde/BuccafurriLR00}. 
    Given a normal program $\prg$ with weak constraints and an atom $a \in \HB$, we can construct an \ASPQW\ program $\qprg$ of the form $\exists \prg: \{\leftarrow \naf a\}$. According to the semantics of \ASPQW, $\qprg$ is coherent if there exists $M\in\OptAS{\prg}$ such that $C \cup \fix{\prg}{M}$ is coherent. 
    By construction, $C \cup \fix{\prg}{M}$ if and only if $a \in M$. Thus, $M \in QAS(\qprg)$ (i.e. $\qprg$ is coherent) if and only if $M \in \OptAS{\prg}$ and $a \in M$.
    
    
    \noindent(Membership) Given an \ASPQW program $\qprg$ of the form $\exists \prg: C$, we can construct an ASP program $\prg^{*}$ such that $\qprg$ is coherent if and only if $\prg^{*}$ admits an optimal answer set $M$ such that $\levelCost{\prg^{*}}{M}{1} = 0$. It is known that the task of verifying the existence of an optimal answer set whose cost is $c$ at level $l$ is $\Delta_2^P$-complete~\citep{DBLP:journals/corr/abs-2401-02153}. 

    Let $\qprg = \exists \prg: C$, we construct (1) the program $\prg'$ obtained by uniformly increasing the level of weak constraints in $\prg$ in such a way that the lowest level is 2; (2) the program $C'$ by translating each strong constraint $r \in C$ is into a normal rule $unsat \leftarrow B_r$, where $unsat$ is a fresh atom not appearing anywhere else.
     
    Let $\prg^* = \prg' \cup C' \cup \{\weakarr unsat [1@1]\}$, from the stratified definition assumption we know that $\heads{C'} \cap at(\prg') = \emptyset$, and so, $U = at(\prg)$ is a splitting set for $\prg^*$ and so each $M \in \AS{\prg^*}$ is of the form $M_1 \cup M_2$ where $M_1 \in \AS{bot_U(\prg^*)}$ and $M_2 \in \AS{e_U(\prg^*\setminus bot_U(\prg^*),M_1)}$. In particular, $bot_U(\prg^*) = \prg'$ and $\prg^*\setminus bot_U(\prg^*) = C'$, and so $M_1 \in \AS{\prg'}$ and $M_2 \in \AS{e_U(C',M_1)}$.
    Thus, from Lemma~\ref{lemma:fix_to_split}, $M \in \AS{C' \cup \fix{\prg'}{M_1}}$.
    
    By construction, weak constraints in $\prg'$ have the highest levels and so, $M$ is an optimal answer set of $\prg^{*}$ if and only if $M_1 \in \OptAS{\prg'}$, and $M \in \AS{C' \cup \fix{\prg'}{M_1}}$.

    Moreover, since $C$ is a stratified program with strong constraints then the incoherence of $C$ can only be caused by strong constraint violations, that are encoded as normal rules in $C'$ defining the fresh atom $unsat$.  
    Thus, $C' \cup \fix{\prg'}{M_1}$ admits always a unique answer $M$. In particular, if $unsat \in M$ then some strong constraints in $C$ are violated and so $C \cup \fix{\prg}{M_1}$ is incoherent, otherwise, no strong constraints in $C$ are violated and so, $C \cup \fix{\prg}{M_1}$ is coherent.

    By construction, $\prg^*$ contains the weak constraint $\weakarr unsat [1@1]$, and so, if there exists $M \in \OptAS{\prg^*}$ such that $unsat \in M$ then $\levelCost{\prg^*}{M}{1} = 1$ and does not exist $M' \in \OptAS{\prg^*}$ such that $unsat \notin M$ with $\levelCost{\prg^*}{M'}{1} = 0$.

    Thus, if there exists $M \in \OptAS{\prg^*}$ such that $unsat \notin M$ then $\prg^*$ admits an optimal answer set that costs $0$ at level 1. Accordingly, since $unsat \notin M$ then every constraint in $C$ is satisfied, and so $\qprg$ is coherent.
    Conversely, if there exists $M \in \OptAS{\prg^*}$ such that $unsat \in M$ then each $M'\in \OptAS{\prg^*}$ contains $unsat$ then $\prg^*$ does not admit an optimal answer set that costs $0$ at level 1.
    Accordingly, since $unsat$ appears in every $M'$ then at least one constraint in $C$ is violated, and so $\qprg$ is incoherent.
\end{proof}

\propEEOrFFisNPorCONP
The result follows trivially from Lemma~\ref{lemma:uni-plain}, once we observe that one application of $col_1$ builds an equi-coherent program with one quantifier. 
%
%

Finally, the suitability of \ASPQW for modeling optimization problems is witnessed by the following.

\begin{lemma}[\cite{DBLP:journals/tcs/Krentel92}]\label{lemma:delta_complete}

    Let $X_1,\ldots,X_n$ be disjoint sets of propositional variables and $\phi$ be a propositional formulas over $X_1,\ldots,X_n$.
    Given a pair of truth assignments $\tau_1,\tau_2$ over a set of variables $X={x_1,\ldots,x_m}$, we say that $\tau_1$ is lexicographically greater that $\tau_2$ if $\tau_1(x_i) = \top$ and $\tau_2(x_i)=\bot$ with $1\leq i\leq m$ being the smallest index for which $\tau_1(x_i)\neq \tau_2(x_i)$.
    Let $\Phi$ be a QBF formula of the form $\forall X_2 \exists X_3 \ldots \mathcal{Q} X_n\ \phi$, where each $Q\in\{\exists,\forall\}$, and $\phi$ is a formula in 3-DNF if $n$ is even, otherwise it is in 3-CNF, and $X_1 = \{x_1,\ldots,x_m\}$.
    Deciding whether the lexicographically minimum truth assignment $\tau$ of variables in $X_1$, such that $\forall X_2 \exists X_3 \ldots \mathcal{Q} X_n\ \phi_{\tau}$ is satisfied (assuming such $\tau$ exists), satisfies the condition $\tau(x_m)=\top$ is a $\Delta_{n+1}^P$-complete problem.
\end{lemma}

\thmWeightAndLevelsIsDelta
%
\begin{proof}
(Hardness) 
Starting from the $\Delta_{n+1}^P$-complete problem introduced by Lemma~\ref{lemma:delta_complete}, we can construct a plain alternating \ASPQW program with $n$ quantifier $\qprg$ such that an atom, namely $x_m$, appears in some optimal quantified answer set of $\qprg$ if and only if the answer to the problem is ``yes".

Without loss of generality we assume that $n$ is even, $\phi$ is propositional formula in 3-DNF, and $X_1 = \{x_1,\ldots,x_m\}$.

For simplicity, we introduce some set of rules that will be used in the construction of $\qprg$.
More precisely, $sat(\phi)$ denotes the set of rules of the form $sat_{\phi} \leftarrow l_1^i, l_2^i, l_3^i$, where $D_i = l_1^i\wedge l_2^i\wedge l_3^i$ is a conjunct in $\phi$; whereas for a set of variables $Z = \{z_1,\ldots,z_k\}$, $choice(Z)$ denotes the program made of a single choice rule of the form $\{z_1;\ldots;z_k\}\leftarrow$.

We are now ready to construct the program $\qprg$ of the form $\exists\prg_1\forall\prg_2\ldots\forall \prg_{n}: C: \wprg$, where $\prg_i = choice(X_i)$, for each $1\leq i\leq n$, and the programs $\cprg$ and $\wprg$ are of the form:

\begin{minipage}{.49\textwidth}
\footnotesize
\smallskip
\[
    \cprg = \left\{
    \begin{array}{rllr}
        &           & sat(\phi) &\\
        & \leftarrow& \naf sat_{\phi}  & \\
    \end{array}
    \right\}
\]    
\smallskip
\end{minipage}
\hfill
\begin{minipage}{.49\textwidth}
\footnotesize
\[
    \wprg = \left\{
    \begin{array}{lr}
        \weakarr x_i\ [1@m-i+1] &\forall 1\leq i \leq m \\
    \end{array}
    \right\}
\]    
\end{minipage}

Intuitively, the program $\prg_1$ is used to guess a possible assignment $\tau$ over variables in $X_1$, for which we want to verify the satisfiability of the QBF formula $\Phi:\forall X_2 \exists X_3 \ldots \forall X_n\ \phi_{\tau}$. 
The following subprograms $\prg_i$, with $2\leq i\leq n$, precisely match the quantifier alternation of $\Phi$ and are used for guessing possible truth assignments for variables in $X_i$.  
Once the final constraint program $\cprg$ is reached, we can evaluate the 3-DNF formula $\phi$ according to the truth assignments guessed by previous subprograms. 
The rules in $sat(\phi)$ will derive the atom $sat$ if there is at least one conjunct in $\phi$ that is satisfied. Finally the last constraint in $C$ impose that at least one conjunct must be satisfied. 

Thus, there exists a quantified answer set of $\qprg$ if and only if there exists a assignment of variables in $X_1$ such that $\Phi$ is satisfiable.
Since the program $\wprg$ contains the set of weak constraints of the form $\weakarr x_i\ [1@m-i+1]$ for each $i \in [1,\ldots,m]$ then the cost of each quantified answer set is given by the true atoms in the guessed $\tau$.
Thus, by assigning the highest priority to the atom $x_1$ (i.e. $m-1+1=m$) and the lowest priority to $x_m$ (i.e. $m-m+1=1$) we can simulate the lexicographical order described above.
In conclusion, the optimal quantified answer set $\Pi$ corresponds to the lexicographically minimum truth assignment $\tau$, such that $\Phi$ is coherent.
By construction $x_m$ is derived if and only if $\tau(x_m)$ is true, and so the thesis follows. 

(Membership) 
According to Theorem 3 of \cite{DBLP:journals/tplp/AmendolaRT19}, we know that the coherence of an existential plain alternating program with $n$ quantifiers falls within the complexity class $\Sigma_{n}^P$-complete. 
By following similar observations employed in the proofs by \cite{DBLP:journals/tkde/BuccafurriLR00,DBLP:journals/ai/SimonsNS02} an optimal solution can be obtained, by implementing a binary search on the value of $k$, with a logarithmic number of calls to an oracle in $\Sigma_{n}^P$ (checking that no better solution than current exists). A final call to the oracle can ensure the existence of an optimal solution containing $a$. 
Since $k$ can be exponential w.r.t. the input size~(\cite{DBLP:journals/tkde/BuccafurriLR00,DBLP:journals/ai/SimonsNS02}) the thesis follows.
\end{proof}
\thmNoLevelCostantWeightIsTheta
%
{\begin{proof}
(Hardness) Let a QBF formula $\Phi$ be an expression of the form $\mathcal{Q}_1 X_1 \ldots \mathcal{Q}_n X_n \phi$, where $X_1,\ldots,X_n$ are disjoint sets of propositional variables, $\mathcal{Q}_i \in \{\exists,\forall\}$ for all $1\leq i\leq n$, $\mathcal{Q}_i \neq \mathcal{Q}_{i+1}$ for all $1\leq i < n$, and $\phi$ is a 3-DNF formula over variables in $X_1,X_2,\ldots,X_n$ of the form $D_1 \vee \ldots \vee D_n$, where each conjunct $D_i = l_1^i\wedge l_2^i\wedge l_3^i$, with $1\leq i\leq n$. 
A $k$-existential QBF formula $\Phi$ is a QBF formula where $n = k$ and $\mathcal{Q}_1 = \exists$.

Given a sequence of { $m$} $k$-existential QBF formulas { $\Phi_1,\ldots,\Phi_m$,} with $k$ being even and greater than or equal to $2$, and such that if $\Phi_j$ is unsatisfiable then also $\Phi_{j+1}$ is unsatisfiable, where $1\leq j < m$, deciding whether $v(\Phi_1,\ldots,\Phi_m) = max\{ j \mid 1\leq j \leq m \wedge \Phi_j\ \textit{is satisfiable}\}$ is odd is $\Theta_{k+1}$-complete~\citep{DBLP:journals/tkde/BuccafurriLR00}.

The above problem can be encoded into an \ASPQW program $\qprg$ such that a literal, namely $odd$, appears in some optimal quantified answer set of $\qprg$ if and only if $v(\Phi_1,\ldots,\Phi_m)$ is odd.
For simplicity, we introduce notation { for some sets of rules} that will be used in the construction of $\qprg$.
More precisely, given a QBF formula $\Phi$, $sat(\Phi)$ denotes the set of rules of the form $sat_{\Phi} \leftarrow l_1^i, l_2^i, l_3^i$, where $D_i = l_1^i\wedge l_2^i\wedge l_3^i$ is a conjunct in $\phi$; whereas for a set of variables $X_i = \{x_1^i,\ldots,x_n^i\}$ in $\Phi$, and an atom $a$, $choice(X_i,a)$ denotes the choice rule $\{x_1^i;\ldots;x_n^i\}\leftarrow a$.
We are now ready to construct the program $\qprg$.

First of all, we observe that all the formulas $\Phi_1,\ldots,\Phi_m$ have the same alternation of quantifiers. Thus, there is a one-to-one correspondence between the quantifiers in the QBF formulas and those in $\qprg$.
Let $\qprg$ be of the form $\Box_1\prg_1\Box_2\prg_2\ldots\Box_k \prg_k: C: \wprg$ where $\Box_i = \exists$ if $\mathcal{Q}_i = \exists$ in a formula $\Phi_j$, otherwise $\Box_i=\forall${ . The program} $\prg_1$ is of the form

{
\footnotesize
\[
    \prg_1 = \left\{
    \begin{array}{rllr}
        \{solve(1);\ldots;solve(m)\}=1 & \leftarrow& & \\
        unsolved(i)& \leftarrow & solve(j) & \forall\ j,i \in [1,\ldots,m] s.t.\ i>j\\
        odd & \leftarrow & solve(j) & \forall j \in [1,\ldots,m] s.t. \textit{ j is odd}\\
        choice(X_1^j,solve(j)) & & & \forall 1\leq j \leq m\\
    \end{array}
    \right\},
\]
\medskip
}

\noindent
{ while,}
for each $2\leq i \leq k$, the program $\prg_i$ is of the form
{
\footnotesize
\[
    \prg_i = \left\{
    \begin{array}{lr}
        choice(X_i^j,solve(j)) & \forall 1\leq j \leq m\\
    \end{array}
    \right\},
\]
}

\noindent
where each $X_i^j$ denotes the set of variables appearing in the scope of the $i$-th quantifier of the $j$-th QBF formula $\Phi_j$. { Finally,} the programs $\cprg$ and $\wprg$ are of the form

\begin{minipage}{.49\textwidth}
\footnotesize
\smallskip
\[
    \cprg = \left\{
    \begin{array}{rllr}
        &           & sat(\Phi_j) & \forall 1\leq j \leq m\\
        & \leftarrow& solve(j),\ \naf sat_{\Phi_j}  & \forall 1\leq j \leq m\\
    \end{array}
    \right\}
\]    
\smallskip
\end{minipage}
\hfill
\begin{minipage}{.49\textwidth}
\footnotesize
\[
    \wprg = \left\{
    \begin{array}{lr}
        \weakarr unsolved(i)\ [1@1,i] &\forall 1\leq i \leq m \\
    \end{array}
    \right\}.
\]    
\end{minipage}

Intuitively, the first choice rule in $\prg_1$ is used to guess one QBF formula, say $\Phi_j$, among the $m$ input ones, for which we want to verify the satisfiability. The guessed formula is encoded with the unary predicate $solve${, whereas, all the following formulas $\Phi_i$, with $i>j$, are marked as unsolved by means of the unary predicate $unsolved$.}

Then, $\prg_1$ contains different rules of the form $odd \leftarrow solve(j)$ for each odd index $j$ in $[1,m]$. 
Thus the literal $odd$ is derived whenever a QBF formula $\Phi_j$ in the sequence $\Phi_1,\ldots,\Phi_m$ is selected (i.e. $solve(j)$ is true) and $j$ is odd.
The remaining part of $\prg_1$ shares the same working principle of the following subprograms $\prg_i$, with $i\geq 2$. 
More precisely, for each QBF formula $\Phi_j$ in the sequence $\Phi_1,\ldots,\Phi_m$, they contain a choice rule over the set of variables quantified by the $i$-th quantifier of $\Phi_j$.
Note that the atom $solve(j)$ in the body of these choice rules guarantees that only one gets activated, and so the activated choice rule guesses a truth assignment for the variables in the $i$-th quantifier of $\Phi_j$.
Similarly, the constraint program $C$ contains, for each QBF formula $\Phi_j$ in the sequence $\Phi_1,\ldots,\Phi_m$, $(i)$ a set of rules that derives an atom $sat_{\phi_j}$ whenever the truth assignment guessed by the previous subprograms satisfies $\phi_j$, and $(ii)$ a strong constraint imposing that is not possible that we selected the formula $\Phi_j$ (i.e. $solve(j)$ is true) and $\phi_j$ is violated (i.e. $sat_{\Phi_j}$ is false).  
Thus, there exists a quantified answer set of $\qprg$ if and only if there exists a formula $\Phi_j$ in the sequence $\Phi_1,\ldots,\Phi_m$ such that $\Phi_j$ is satisfiable.
Since the program $\wprg$ contains the set of weak constraints of the form {$\weakarr unsolved(j)\ [1@1,j]$} for each $j \in [1,\ldots,m]$, the cost of each quantified answer set is given by the index $j$ of the selected formula. 
Thus, by minimizing the number of unsolved formulas
we are maximizing the index of the satisfiable formula $\Phi_j$. Thus, an optimal quantified answer set corresponds to a witness of coherence for a formula $\Phi_j$, s.t. for each $\Phi_{j'}$, with $j'>j$, $\Phi_{j'}$ is unsatisfiable.
By construction $odd$ is derived whenever $j$ is odd and so the { hardness} follows. 

(Membership) 
According to Theorem 3 of \cite{DBLP:journals/tplp/AmendolaRT19}, we know that the coherence of an existential plain alternating program with $n$ quantifiers falls within the complexity class $\Sigma_{n}^P$-complete.
By following { an} observation employed in the proofs by \cite{DBLP:journals/tkde/BuccafurriLR00}, { the cost of} an optimal solution can be obtained by 
binary search 
{ that terminates in} a logarithmic, { in the value of the maximum cost,} number of calls to an oracle in $\Sigma_{n}^P$ {that checks whether a quantified answer set with a lower cost with respect to the current estimate of the optimum exists.} { Once the cost of an optimal solution is determined, one more call to the oracle (for an appropriately modified instance), allows one to decide} the existence of an optimal solution containing $a$. Since each weak constraint has the same weight and the same level, then we can consider as the maximum cost the number of weak constraint violations. Thus, the number of oracle calls is at most logarithmic in the size of the problem and 
the membership follows.
\end{proof}}

\end{document}